\documentclass[]{ieeeaccess}

\usepackage[utf8]{inputenc}
\usepackage[T1]{fontenc}
\usepackage{cite}
\usepackage{amsmath, amssymb, amsfonts, amsthm, amsbsy}
\usepackage{mathtools}
\usepackage{algorithmic}
\usepackage{graphicx}
\usepackage{realboxes}
\usepackage{enumitem}
\usepackage{textcomp}
\usepackage{bbm}
\usepackage{bm}
\usepackage{multirow}
\usepackage{booktabs}
\usepackage[caption=false, font=footnotesize]{subfig}
\usepackage[dvipsnames, table]{xcolor}  
\usepackage{hyperref}
\usepackage[bottom]{footmisc}

\newcommand{\x}[1]{x_\mathrm{#1}}       
\newcommand{\Frechet}{Fr\'echet }        
\newcommand{\F}{\mathcal{F}}            
\newcommand{\conv}{\mathrm{conv}}

\DeclareMathOperator*{\len}{len}                

\DeclareMathOperator*{\argmin}{arg\,min}
\DeclareMathOperator*{\closure}{cl}
\DeclareMathOperator*{\interior}{int}
\renewcommand{\emptyset}{\text{\O}}             
\renewcommand{\epsilon}{\varepsilon}

\makeatletter
	\AtBeginDocument{
		\DeclareMathVersion{bold}
		\SetSymbolFont{operators}{bold}{T1}{times}{b}{n}
		\SetSymbolFont{NewLetters}{bold}{T1}{times}{b}{it}
		\SetMathAlphabet{\mathrm}{bold}{T1}{times}{b}{n}
		\SetMathAlphabet{\mathit}{bold}{T1}{times}{b}{it}
		\SetMathAlphabet{\mathbf}{bold}{T1}{times}{b}{n}
		\SetMathAlphabet{\mathtt}{bold}{OT1}{pcr}{b}{n}
		\SetSymbolFont{symbols}{bold}{OMS}{cmsy}{b}{n}
		\renewcommand\boldmath{\@nomath\boldmath\mathversion{bold}}
	}
\makeatother

\def\BibTeX{{\rm B\kern-.05em{\sc i\kern-.025em b}\kern-.08em
		T\kern-.1667em\lower.7ex\hbox{E}\kern-.125emX}}

\theoremstyle{definition}
\newtheorem{definition}{Definition}
\newtheorem*{definition*}{Definition}

\newtheorem{proposition}{Proposition}

\newtheorem{example}{Example}
\newtheorem{lemma}{Lemma}
\newtheorem{construction}{Construction}

\newenvironment{Tabular}[2][1]
{\def\arraystretch{#1}\tabular{#2}}
{\endtabular}
\newcommand\TstrutSmall{\rule{0pt}{2ex}}            
\newcommand\BstrutSmall{\rule[-1.5ex]{0pt}{0pt}}    

\hypersetup{
	linkcolor  = NavyBlue,
	citecolor  = NavyBlue,
	colorlinks = true
}

\begin{document}

\history{Date of publication xxxx 00, 0000, date of current version xxxx 00, 0000.}

\doi{00.0000/ACCESS.2024.0000000}

\title{Entanglement Definitions for Tethered Robots: Exploration and Analysis}

\author{
	\uppercase{Gianpietro Battocletti}\textsuperscript{\href{https://orcid.org/0009-0004-6981-0017}{\hspace{0.05cm}\includegraphics[scale=0.012]{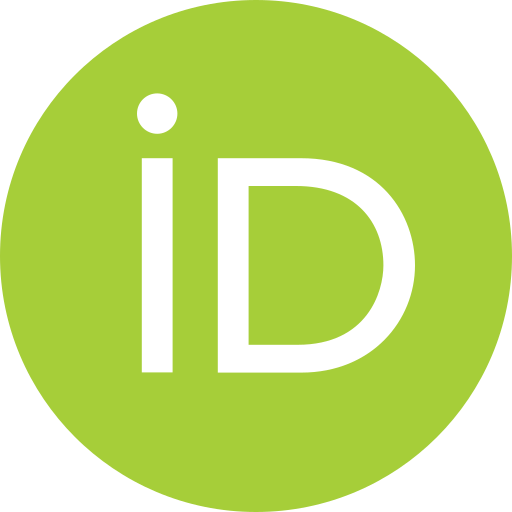}}\hspace{0.05cm}}\authorrefmark{1},
	\uppercase{Dimitris Boskos}\textsuperscript{\href{https://orcid.org/0000-0003-0287-9197}{\hspace{0.05cm}\includegraphics[scale=0.012]{orcid_logo.png}}\hspace{0.05cm}}\authorrefmark{1},
	\uppercase{Domagoj Toli\'{c}}\textsuperscript{\href{https://orcid.org/0000-0002-0988-889X}{\hspace{0.05cm}\includegraphics[scale=0.012]{orcid_logo.png}}\hspace{0.05cm}}\authorrefmark{2,3},
	\uppercase{Ivana Palunko}\textsuperscript{\href{https://orcid.org/0000-0001-5204-2294}{\hspace{0.05cm}\includegraphics[scale=0.012]{orcid_logo.png}}\hspace{0.05cm}}\authorrefmark{2}, 
	and \uppercase{Bart De Schutter}\textsuperscript{\href{https://orcid.org/0000-0001-9867-6196}{\hspace{0.05cm}\includegraphics[scale=0.012]{orcid_logo.png}}\hspace{0.05cm}}\authorrefmark{1}, \IEEEmembership{Fellow,~IEEE}
}

\address[1]{Delft Center for Systems and Control, Delft University of Technology, 2628 CN Delft, The Netherlands.}
\address[2]{LARIAT – Laboratory for Intelligent Autonomous Systems, University of Dubrovnik, Croatia.}
\address[3]{RIT Croatia, Dubrovnik, Croatia.}
	
\corresp{Corresponding author: Gianpietro Battocletti (e-mail: g.battocletti@tudelft.nl).}

\tfootnote{This publication has been supported by funding from the European Union’s Horizon 2020 Research and Innovation Programme under grant agreement No 871295 (SeaClear) and by funding from the European Union's Horizon Europe Programme under grant agreement No 101093822 (SeaClear 2.0).}

\markboth
{G. Battocletti \headeretal: Entanglement Definitions for Tethered Robots: Exploration and Analysis}
{G. Battocletti \headeretal: Entanglement Definitions for Tethered Robots: Exploration and Analysis}

\begin{abstract}
	In this article we consider the problem of tether entanglement for tethered mobile robots. 
	One of the main risks of using a tethered connection between a mobile robot and an anchor point is that the tether may get entangled with the obstacles present in the environment or with itself.
	To avoid these situations, a non-entanglement constraint can be considered in the motion planning problem for tethered robots. 
	This constraint is typically expressed as a set of specific tether configurations that must be avoided.
	However, the literature lacks a generally accepted definition of entanglement, with existing definitions being limited and partial in the sense that they only focus on specific instances of entanglement.
	In practice, this means that the existing definitions do not effectively cover all instances of tether entanglement.
	Our goal in this article is to bridge this gap and to provide new definitions of entanglement, which, together with the existing ones, can be effectively used to qualify the entanglement state of a tethered robot in diverse situations.
	The new definitions find application in motion planning for tethered robots, where they can be used to obtain more safe and robust entanglement-free trajectories. 
\end{abstract}


\titlepgskip=-21pt

\maketitle

\section{Introduction}
\label{sec:introduction}
\PARstart{T}{ethered} mobile robots are a class of mobile robots characterized by a cabled connection of the robot with an anchor point, or with another robot \cite{yuh2000design}. 
Tethered mobile robots (from now on referred to as ‘tethered robots’ for brevity) find application in a large number of tasks, such as exploration, inspection, or maintenance, and they are employed in ground \cite{mcgarey2017tslam, krishna1997tethering, mcgarey2018developing}, underwater, \cite{mccammon2017planning, cao2023neptune}, aerial \cite{pratt2008tethered, martinez2021optimization, tognon2021physical}, and space \cite{cosmo1997tethers, wen2008advances, yu2018review} applications. 
These tasks are typically addressed either using remotely operated vehicles (ROVs) or unmanned autonomous vehicles (UAVs). 
In both cases, the cabled connection of a robot to an anchor point can be used as a power source, a communication channel, a lifeline to retrieve the robot in case of malfunctioning, and for accessing additional computational power \cite{yuh2000design, shnaps2014online}. 
However, these advantages do not come without challenges. The tether exerts an external force on the robot, due to gravity, drag, and inertia acting on it \cite{eidsvik2016time}, and it limits the reachable workspace, due to its finite length \cite{igarashi2010homotopic}. 
In addition, the tether may get entangled with obstacles present in the environment or with itself, forming knots. In case of multi-robot systems, this problem is amplified as the tethers of different robots can also get entangled with each other \cite{cao2023neptune, cao2023path}.
\par
Broadly speaking, entanglement occurs when the movement of a tethered robot is restricted due to the physical interaction of the tether with other objects in the environment \cite{cao2023neptune}.
Since this condition is, in general, disadvantageous or even dangerous for a tethered robot, it is important to avoid it during the robot's motion. This is typically achieved through the addition of a non-entanglement constraint.
Non-entanglement constraints capture the fact that some tether configurations, despite being achievable by a tethered robot within the maximum tether length, may hinder the mobility of the robot, requiring it to perform specific operations to recover full motion capabilities.
To be able to consistently prevent entanglement, we first need a definition that captures the occurrence of entanglement and that is measurable. 
Despite the interest in motion planning for tethered robots, entanglement has not been studied extensively in the literature. 
The existing definitions are generally limited and partial, as they focus only on specific instances of entanglement and specific applications. Moreover, the existing definitions often require specific assumptions on the tether and on the environment, which limits the applicability of those definitions.
\par
The main application of the entanglement definitions is to use them to add a non-entanglement constraint to the motion planning problem for tethered robots. Several works that use non-entanglement constraints in the motion planning problem for tethered robots already exist in the literature \cite{shapovalov2020tangle, cao2023path, hert1999motion, zhang2019planning, hert1997planar}. 
However, the resulting motion planning algorithms are often application-specific and tailored to a specific entanglement definition. The introduction of more general entanglement definitions can help in the development of more versatile entanglement-aware motion planning algorithms for tethered robots. 
In the example of Figure \ref{fig:example_motion_planning}, a non-entanglement constraint can be used to exclude one or more of the possible motion paths by evaluating the entanglement state of the tether configuration that would result from the robot moving along such paths.
A more in-depth study of the properties of the entanglement definitions can also lead to more efficient entanglement-aware motion planning algorithms for tethered robots.
For instance, the entanglement definitions can be used to compute a safe set in which the robot must stay to maintain the tether in a non-entangled state. Such a set can then be used as the domain for an existing motion planning algorithms for tethered robots, resulting in safety guarantees in terms of entanglement avoidance.
\begin{figure}[t]
	\centering
	\includegraphics[]{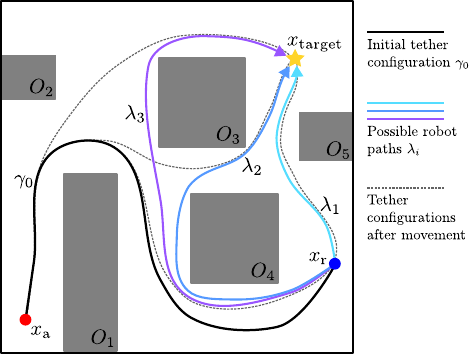}
	\caption{Example of a motion planning problem for a tethered robot in a 2D environment. The robot must reach the location $\bm{\x{target}}$ from its current location $\bm{\x{r}}$. Three possible paths $\bm{\lambda_1, \lambda_2, \lambda_3}$ are depicted in the image. Given the initial tether configuration $\bm{\gamma_0}$, the three possible paths would result in three different tether configurations after the motion of the robot, which are represented by the gray dashed lines. 
	Each of the resulting tether configuration can have a different entanglement state depending on the entanglement definition being used.}
	\label{fig:example_motion_planning}
\end{figure}
\par
In this work we provide and analyze a broad set of new entanglement definitions. 
After reviewing the existing entanglement definitions, we expand and generalize them by considering new entanglement definitions that can be used to complement the existing ones in the characterization of the entanglement state of a tether.
We analyze the properties of the definitions to highlight their advantages, disadvantages, and specific characteristics. 
In particular, we investigate the restrictions that these definitions introduce on the workspace of a tethered robot.
We also compare the definitions with each other and with the existing ones, highlighting the relations between them.
Finally, we validate the definitions empirically by comparing them against the opinion of experts in the field of tethered robotics using a survey.
\par
We remark that, in order to keep this article more focused, in this work we focus exclusively on the presentation, analysis, and comparison of the entanglement definitions. 
The actual application of these definitions to the motion planning problem will be addressed in our upcoming works.
\par
In summary, the contributions of this work are:
\begin{itemize}
	\item an extensive review and classification of the existing entanglement definitions;
	\item the introduction of a broad set of new entanglement definitions to expand and supplement the existing ones. Unlike most of the existing entanglement definitions, the new definitions are agnostic to the type of robot, environment, or tether considered, and can be applied both in 2D and 3D environments, to single-robot and multi-robot systems, to loose and slack tethers, and to fixed-length and variable-length tethers;
	\item the analysis of the properties of the proposed definitions, and the characterization of the workspace of a tethered robot under different entanglement definitions;
	\item a formal analysis of the relations between all the entanglement definitions, where we highlight the connections and the mutual relations between them;
	\item an empirical validation of the entanglement definitions against the opinion of experts in the field of tethered robotics.
\end{itemize}
\par
The article is organized as follows. 
Section~\ref{sec:related_work} contains a review of the existing literature on motion planning for tethered robots, with particular attention to the approaches used to define entanglement in previous works. 
Preliminary concepts and the problem statement are introduced in Section~\ref{sec:preliminaries}.
In Section~\ref{sec:definitions} the new entanglement definitions are stated. 
The properties of the proposed entanglement definitions are investigated in Section~\ref{sec:properties}, while the relations between the proposed definitions and the existing ones are analyzed in Section~\ref{sec:relationships}. 
The analysis of the entanglement definitions is concluded in Section~\ref{sec:validation} with their empirical validation by experts in the field of tethered robotics.
Conclusions and open points are given in Section~\ref{sec:conclusions}. 
The proofs of the results presented throughout the article are collected in Appendices \ref{appendix:proofs_definitions}, \ref{appendix:proofs_properties}, \ref{appendix:proofs_comparison}.

\section{Related work}
\label{sec:related_work}
In the last decades, significant attention has been devoted to the motion planning problem for tethered robots. A variety of different approaches has been investigated to tackle this problem \cite{kim2014path, bhattacharya2021search, brass2015shortest, narayanan2013planning, yang2022efficient_b}. 
Despite the large number of works on the topic, most of the works in the literature on tethered mobile robots do not consider any entanglement constraint, and focus only on finding a feasible solution for the motion planning problem under the geometric constraints posed by the presence of the tether \cite{kim2014path, sinden1990tethered, xavier1999shortest, bhattacharya2012identification, teshnizi2021motion , kim2015path}.
In the works were entanglement has been considered, it has typically been considered only from limited application-specific perspectives.
The majority of those works can be grouped into a few distinct categories depending on how they define and manage entanglement. 
The resulting categories are comprised of works that (i) consider entanglement as the contact between the tether and obstacles; (ii) consider entanglement as the contact between the tethers of different robots (in multi-robot systems); (iii) consider entanglement as the tether looping around obstacles.
Most of the works in the literature regard single-agent systems. Some works (in particular those in the third category) consider multi-robot systems instead, where several robots share the same environment.
\par
The works in the first category define entanglement as the contact between the tether and an obstacle. 
In \cite{rajan2016tether, teshnizi2014computing, mcgarey2016line} a robot connected to an anchor point via a taut tether in a 2D obstacle-rich environment is considered. Entanglement is not strictly prohibited, and its detection is used instead to keep track of the tether configuration, and to obtain information on the location of the obstacles.
In \cite{petit2022tape, martinez2021optimization} the same entanglement definition is applied to an aerial drone moving in a 3D environment. 
There the goal is to plan the motion of the robot and the variable\footnote{In some works such as  \cite{petit2022tape, abadmanterola2011motion, shapovalov2020tangle} the length of the tether is varied over time in order to keep it always taut by using a winch or a dedicated tether length control system.} length of the tether in order to avoid any contact between the taut tether and the obstacles in the environment.
This definition of entanglement is used also in \cite{almhdawi2021carti}, where entanglement is avoided through the use of micro-thrusters placed along the tether which allow to actively control its shape and keep it away from obstacles.
\par
Works in the second category consider multi-robot systems where entanglement is defined as the interaction between the tethers of two different robots. 
In \cite{sinden1990tethered, hert1996ties, zhang2019planning} a multi-robot system in an obstacle-free 2D environment is considered. The tethers of all robots are kept taut, so entanglement happens when a bend is formed in a tether.
A centralized path planning algorithm is tasked with finding a set of trajectories for the robots that avoid intersections between the tethers of the robots. When avoiding entanglement is impossible, the planner returns instead a motion strategy that minimizes the number of tethers getting entangled. 
The algorithm is extended to the 3D case in \cite{hert1999motion, patil2023coordinating}. 
A more general setting is tackled in \cite{cao2023neptune, cao2023path}, where a multi-robot system with slack tethers in a 3D environment is considered. \cite{cao2023path} focuses on providing non-entanglement guarantees for a large number of robots in a crowded, obstacle-free environment.
In \cite{cao2023neptune} a decentralized algorithm is used to compute kinodynamically feasible paths in the presence of obstacles. 
To identify tether configurations that are at risk of entanglement, \cite{cao2023neptune, cao2023path} analyze their topological properties through the computation of a signature\footnote{More details on the tether signature are provided in Section \ref{subsec:homotopy} and in Appendix \ref{appendix:proofs_definitions}.} of each tether configuration, and relate specific patterns in the signatures to the entanglement state of the tethers.
\par
The last group of works considered here defines entanglement as the looping of the tether around obstacles. In \cite{shapovalov2020tangle, mechsy2017novel, sharma20192approximation} a robot moving in a 2D obstacle-rich environment is considered, and any tether configuration that does a full loop around an obstacle is considered to be entangled and is therefore avoided.\footnote{A similar constraint is introduced for computational reasons also in other articles on motion planning for tethered robots such as \cite{kim2014path, bhattacharya2010search, salzman2015optimal}. However, therein the constraint is not intended explicitly for entanglement avoidance, so we do not formally consider those works as part of this category.}
A variation of this definition is considered in several works \cite{abadmanterola2011motion, tanner2013online, mahmud2021multi, mcgarey2017tslam} that focus on exploration and inspection tasks in 2D obstacle-rich environments, where a number of waypoints must be visited by a tethered robot before returning to the starting location. 
Since in those works the robot always returns to the initial location at the end of its motion, they consider only closed tether configurations, where the initial point and end point coincide. The non-entanglement constraint is then stated in the form of a non-looping constraint for the whole tether. This means that at the end of the motion there must be no obstacle being encircled by the tether \cite{mccammon2017planning}.

\section{Preliminaries}
\label{sec:preliminaries}
Some preliminary definitions are introduced first to facilitate the exposition of the entanglement definitions.

\subsection{Workspace topology}
\label{subsec:workspace_and_paths}
Let the workspace $\mathcal{X}$ be an open convex subset of $\mathbb{R}^n$ with $n\in\{2, 3\}$, and let $\{O_i\}_{i=1,\ldots, m}$ be a finite set of disjoint closed obstacles having a non-empty simply-connected interior and without degenerate boundary\footnote{An obstacle does not have a degenerate boundary if for any point of the boundary there is an arbitrarily close interior point \cite{grigoriev1998polytime}.} \cite{lavalle2006planning}.
We indicate with $\mathcal{O} = \cup_{i=1}^m O_i$ the obstacle region, corresponding to the part of $\mathcal{X}$ that is covered by obstacles. We indicate with $\partial\mathcal{O}$ the boundary of the obstacle region, and with $\interior\mathcal{O}$ its interior. 
The free workspace is defined as\footnote{We remark that the closure operator is introduced in the definition of the free workspace to make sure that the shortest path between any two points in $\mathcal{X}_\text{free}$ exists (see Lemma \ref{lemma:existence_shortest_path}) \cite{lavalle2006planning, latombe1991robot}.} $\mathcal{X}_\text{free} = \closure(\mathcal{X} \setminus \mathcal{O})$ \cite{lavalle2006planning}.
We assume that $\mathcal{X}_\text{free}$ is formed by a single path-connected component. If this condition were not true, there would be locations in the free space that could not be reached from a given starting point through a continuous path. In that case, the unreachable locations would be considered as parts of the obstacle region.  
\par
A \emph{path} $\gamma$ is a continuous function $\gamma:[0,1]\rightarrow\mathcal{X}$ \cite{hatcher2005algebraic}. In the rest of the current article, with a slight abuse of notation, $\gamma$ will be used to indicate both the image and the function that generates it. 
We denote by $\len(\gamma)$ the length of a path $\gamma$. In the following, we only consider finite-length paths.
A path $\gamma$ is said to be \emph{obstacle-free} if $\gamma(s)\in\mathcal{X}_\text{free}, \forall s\in[0,1]$.
The points $\gamma(0)$ and $\gamma(1)$ are called respectively initial point and terminal point of the path, or, collectively, endpoints of the path.
A reparametrization of a path $\gamma$ is a path $\gamma\circ\alpha$, where $\alpha$ is a continuous, non-decreasing surjective mapping $\alpha: [0,1]\rightarrow[0,1]$ with $\alpha(0)=0$ and $\alpha(1)=1$ \cite{chambers2010homotopic}. 
We indicate with $\mathcal{A}$ the set of all the mappings $\alpha:[0,1]\rightarrow[0,1]$ having these properties.
Given two paths $\gamma_1$, $\gamma_2$ such that $\gamma_1(1) = \gamma_2(0)$, their concatenation $\gamma_3 = \gamma_1 \diamond \gamma_2$ is the new path defined by
\begin{equation*}
	\gamma_3(s) =
	\begin{cases}
		\gamma_1(2s), & 0 \leq s \leq \frac{1}{2}, \\
		\gamma_2(1-2s), & \frac{1}{2} \leq s \leq 1.
	\end{cases}
\end{equation*}
For any path $\gamma$ from $x_1$ to $x_2$ we define the reverse path $\gamma^\text{reverse}(s) = \gamma(1-s)$, i.e., the path going from $x_2$ to $x_1$ \cite{hatcher2005algebraic}. The convex hull of a path $\gamma$ is defined as
\begin{equation*}
	\conv({\gamma}) = \left\{\sum_{i=1}^n a_i x_i : a_i\in\mathbb{R}, \sum_{i=1}^n a_i = 1, a_i \geq 0, x_i \in \gamma \right\}.
\end{equation*}
\par
We define now some relevant paths that will be used in the rest of this article. The first is the \emph{straight line segment} between two points. Given two points $x_1, x_2 \in \mathcal{X}$, the straight line segment $l_{x_1, x_2}:[0,1]\rightarrow\mathcal{X}$ between $x_1$ and $x_2$ is defined by $l_{x_1, x_2}(s) = (1-s) x_1 + s x_2$. 
Next, we introduce the restriction of a path $\gamma$ to an interval $[s_1, s_2]$. Given a path $\gamma$ and an interval $[s_1, s_2] \subseteq [0,1]$, we define the path $\gamma_{[s_1, s_2]}:[0,1]\rightarrow\mathcal{X}$ by $\gamma_{[s_1, s_2]}(s) = \gamma(s_1 + s(s_2-s_1))$, which corresponds to the part of the path $\gamma$ between $\gamma(s_1)$ and $\gamma(s_2)$.
We also define the notion of a shortest path among a set of paths: given a set of paths $S$ we denote a shortest path in $S$ by
\begin{equation*}
	\gamma^*_S \in \argmin_{\gamma\in S} \left(\,\len(\gamma)\right),
\end{equation*}
provided that such a path exists \cite{grigoriev1998polytime, hershberger1994computing}. 
We denote a globally shortest path between two points as $\gamma^*_{x_1, x_2} \coloneqq \gamma^*_{\Gamma_{x_1, x_2}}$,
where $\Gamma_{x_1, x_2}$ is the set of all paths between $x_1$ and $x_2$, defined as 
\begin{equation*}
	\Gamma_{x_1, x_2} = \{\gamma:[0,1]\rightarrow\mathcal{X}_\text{free}:\gamma(0) = x_1, \gamma(1) = x_2\}.
\end{equation*}
Given the topological properties of $\mathcal{X}_\text{free}$, there always exists a shortest path between two points, as formalized in Lemma \ref{lemma:existence_shortest_path} (see Appendix \ref{appendix:proofs_definitions}).
Finally, we define a \emph{taut path} as a locally shortest path. Formally, a path $\gamma$ is called taut if for all $s\in[0,1]$ there exists an interval $J_s$ containing a neighborhood of $s$ such that $\gamma_{J_s}$ is a shortest path \cite{burago2001course}.

\subsection{Paths and homotopy equivalence}
\label{subsec:homotopy}
We focus now on paths lying in $\mathcal{X}_\text{free}$.
In a topological space, the presence of obstacles (or `punctures') gives rise to multiple homotopy classes in which a path can lie. Informally, two paths $\gamma_1, \gamma_2$ are said to belong to the same homotopy class if they can be continuously transformed into one another without crossing any obstacle. Three types of homotopic equivalence relationships are considered in this work:

\subsubsection{Free Homotopy}
Two paths $\gamma_1, \gamma_2$ are said to be \emph{freely homotopic}\footnote{In the literature the word `freely' is often omitted when defining this type of homotopic relation. We use it to distinguish it more clearly from the other two types of homotopy, following the naming convention used in \cite{lee2010introduction}.} (or to belong to the same \emph{free homotopy class}) if there exists a continuous function $H:[0,1]\times[0,1]\rightarrow\mathcal{X}_\text{free}$ such that $H(s,0)=\gamma_1(s), H(s,1) = \gamma_2(s), \forall s\in[0,1]$ \cite{hatcher2005algebraic}. Being freely homotopic is an equivalence relation on the set of all continuous maps from $[0,1]$ to $\mathcal{X}_\text{free}$ \cite{lee2010introduction}. 
We indicate the existence of a freely homotopic relation between two paths with $\gamma_1\simeq\gamma_2$ \cite{lee2010introduction}.
The set of all the free homotopies between two curves $\gamma_1, \gamma_2$ is indicated with $\mathcal{H}_{\gamma_1, \gamma_2}$.
Given a homotopy $H$, we denote by $H(s, \cdot):[0,1] \rightarrow \mathcal{X}_\text{free}$ the path $t \mapsto H(s,t)$ that goes from the point $\gamma_1(s)$ to the point $\gamma_2(s)$.

\subsubsection{Path Homotopy}
Two paths $\gamma_1, \gamma_2$ with the same endpoints, i.e., such that $\gamma_1(0) = \gamma_2(0) = x_1$ and $\gamma_1(1) = \gamma_2(1) = x_2$, are said to be \emph{path-homotopic} if there exists a continuous mapping $H:[0,1]\times [0,1]\rightarrow\mathcal{X}_\text{free}$ such that $H(s,0)  = \gamma_1(s), H(s,1) = \gamma_2(s), \forall s\in [0,1]$, and $H(0,t)=x_1$, $H(1,t)=x_2, \forall t\in [0,1]$ \cite{hatcher2005algebraic}. 
Given a pair of fixed points $x_1$ and $x_2$, path homotopy is an equivalence relation on $\Gamma_{x_1, x_2}$ \cite{lee2010introduction}.
We indicate the existence of a path homotopy relationship between two paths with $\gamma_1 \sim \gamma_2$ \cite{lee2010introduction}. 
We indicate by $\Gamma_{x_1, x_2}/\sim$ the set of all path homotopy classes between the points $x_1$ and $x_2$, and by $[\gamma]$ the path homotopy class of a path $\gamma$ \cite{lee2010introduction}.
We denote the path homotopy between a path $\gamma$ and a point $x_0$ with $\gamma\sim x_0$, where $x_0$ indicates a constant map from the interval $[0,1]$ to the point $x_0$ \cite{lee2010introduction}.
\par
A common approach to determine the path homotopy class of a path consists in computing its \emph{signature}.
In 2D, a signature is a function $\mathbbm{h}:\Gamma_{x_1, x_2}\rightarrow \Gamma_{x_1, x_2}/\sim$ that maps every path $\gamma$ to its path homotopy class \cite{bhattacharya2018path}.
For more details on the constructions of signatures see \cite{bhattacharya2012topological, bhattacharya2018path, hatcher2005algebraic}, and Construction \ref{construction:signature} and Proposition \ref{prop:signature_is_invariant} in Appendix \ref{appendix:proofs_definitions}.

\subsubsection{Relative homotopy}
The third type of homotopy equivalence considered is that of \emph{relative homotopy} \cite{chambers2010homotopic}.\footnote{
	We largely base the definition of relative homotopy on \cite{chambers2010homotopic}. A definition of relative homotopy is given also in \cite{lee2010introduction}, but it is different from the one we use. 
	We also want to remark that, while we use a number of definitions and concepts from \cite{chambers2010homotopic}, the notation and the naming convention has been changed. This has been done mainly to uniformize the notation with the other homotopy definitions. 
	In particular, we do not look for a \emph{leash map} between two \emph{leashes}, but for a \emph{relative homotopy map} between two \emph{paths}.}
In this case we consider paths with distinct endpoints, and we require the homotopic transformation between the two paths to make the endpoints move along pre-specified paths.
Let $\gamma_1, \gamma_2$ be two paths, $\lambda$ be a path from $\gamma_1(0)$ to $\gamma_2(0)$, and $\lambda'$ be a path from $\gamma_1(1)$ to $\gamma_2(1)$. The paths $\gamma_1$ and $\gamma_2$ are said to be \emph{relatively homotopic along $\lambda, \lambda'$} if there exists a continuous mapping $H_{\lambda,\lambda'}:[0,1]\times [0,1]\rightarrow\mathcal{X}_\text{free}$ such that $H_{\lambda,\lambda'}(s, 0) = \gamma_1(s), H_{\lambda,\lambda'}(s, 1) = \gamma_2(s), \forall s \in [0,1]$, and that $H_{\lambda,\lambda'}(0, t) = \lambda(\alpha(t)), H_{\lambda,\lambda'}(1, t) = \lambda'(\alpha'(t)), \forall t \in [0,1]$, where $\alpha, \alpha' \in \mathcal{A}$ are, respectively, reparametrizations of $\lambda$ and $\lambda'$.
For any fixed pair of paths $\lambda, \lambda'$, being relatively homotopic along $\lambda$ and $\lambda'$ is an equivalence relation which we indicate  as $\gamma_1\simeq_{\lambda, \lambda'}\gamma_2$.

\subsection{Homotopic \Frechet Distance}
The homotopic \Frechet distance is a metric suitable to measure the distance between two paths \cite{alt2004comparison}. The homotopic \Frechet distance, which is a variation of the \Frechet distance \cite{alt1995computing}, measures the distance between two obstacle-free paths $\gamma_1$, $\gamma_2$ while taking into account the presence of obstacles between them \cite{chambers2010homotopic}.
More precisely, the homotopic \Frechet distance computes the distance between two paths $\gamma_1, \gamma_2$ as the maximum length of a path $H(s,\cdot)$ that belongs to a free homotopy $H$ between a reparametrization of $\gamma_1$ and a reparametrization of $\gamma_2$. The homotopic \Frechet distance is defined as
\begin{equation*}
	\bar{\F} (\gamma_1, \gamma_2) \coloneqq \inf_{H\in\mathcal{H}_{\gamma_1, \gamma_2}} \left[ \max_{s\in [0,1]} \len \left( H(s, \cdot) \right) \right],
\end{equation*}
where we recall that $\mathcal{H}_{\gamma_1, \gamma_2}$ is the set of all the free homotopies between the paths $\gamma_1$ and $\gamma_2$.

\subsection{Problem formulation}
\label{sec:problem_formulation}
Let $\mathcal{X}, \mathcal{O}, \mathcal{X}_\text{free}$ be respectively the workspace, the obstacle region, and the free workspace where a tethered robot operates. 
The tether is represented by a finite-length obstacle-free path $\gamma$ called \emph{tether configuration}. 
The tether initial point $\x{a} = \gamma(0)$ represents the anchor point, while the tether terminal point $\x{r} = \gamma(1)$ represents the robot location.
Given a tether configuration $\gamma$, we seek to determine its entanglement state.
In particular, we aim to obtain a set of general entanglement definitions that do not depend on the properties of the tether nor on the properties of the environment, and that align with the human intuition of what entanglement is.
\par
The proposed definitions work both for single-robot and multi-robot systems. 
In case of multi-robot systems, a set of $q$ robots $\mathcal{I} = \{1,2,\ldots, q\}$ is considered. 
The robots are located in distinct points $x_{\text{r}, i}, i\in\mathcal{I}$, and their tethers have distinct anchor points $x_{\text{a}, i}, i\in\mathcal{I}$.
For the sake of simplicity, it is assumed that the tethers do not intersect with each other, i.e., $\gamma_i \cap \gamma_j = \emptyset, \forall i \ne j$.
The entanglement state is determined for a single robot $i\in\mathcal{I}$. From the point of view of robot $i$ the other tethers $\gamma_j, j\in\mathcal{I}\setminus \{i\}$ are seen as obstacles; specifically, the obstacle region for robot $i$ is $\mathcal{O}_i = \mathcal{O} \cup \bigcup_{j=1, j\neq i}^q O_\epsilon(\gamma_j)$ for some $\epsilon > 0$, where $O_\epsilon(\gamma_j)$ is the $\epsilon$-inflation\footnote{The $\epsilon$-inflation of a path $\gamma_j$ is the set $O_\epsilon(\gamma_j)=\{x:\text{dist}(x, \gamma_j)\leq\epsilon, x\in\mathcal{X}\}$. 
We introduce this representation of the tethers of the other robots in order to unify the treatment of all the entanglement definitions, and to make the definition of $\mathcal{X}_\text{free}$ consistent with the use of homotopy transformations. In fact, if the tether of another robot were represented solely as a curve, it would have no effect on $\mathcal{X}_\text{free}$ due to the use of the $\text{cl}(\cdot)$ operator in its definition, and thus it would be ignored when defining homotopic transformations \cite{grigoriev1998polytime}.} of the tether~$\gamma_j$~\cite{grigoriev1998polytime}.
In multi-robot systems, a tether is then considered not entangled according to some non-entanglement definition if there exists an $\epsilon>0$ for which the tether is not entangled under the corresponding $\epsilon$-inflation of the other tethers.
\par
It is worth mentioning here that in the literature a difference is sometimes made between taut and slack tethers. The definitions proposed in this article can handle both types of tether configurations and we will typically not distinguish between them.
In addition, all the proposed definitions can handle both a 2D and a 3D workspace.
Following an assumption commonly used in the literature, in 2D the tether is allowed to intersect with itself, forming loops, which in practice allows the robot to cross its own tether.

\section{Non-entanglement definitions}
\label{sec:definitions}
In this section we present the new non-entanglement definitions.
The proposed definitions provide a criterion to determine when a given tether configuration is not entangled, which is why from now on we will refer to the definitions as \emph{non-entanglement definitions}.\footnote{The reason for this choice versus considering \emph{entanglement definitions} is that, in most of the problems in which we plan to use the definitions, the goal is to maintain a non-entangled tether configuration. Therefore, we state the definitions using conditions that identify a non-entangled tether configuration, so they can be applied directly in the form in which they are presented here.} 
We start this section with the review of the definitions available in the literature (Section \ref{subsec:existing_definitions}), which is organized following the categories that have been introduced in Section \ref{sec:related_work}. 
We then proceed to present the proposed non-entanglement definitions (Section \ref{subsec:proposed_definitions}). 
We conclude this section by discussing a relaxation of the non-entanglement definitions (Section \ref{subsec:delta_relaxation}).

\subsection{Existing non-entanglement definitions}
\label{subsec:existing_definitions}
The non-entanglement definitions found in the literature are reviewed here. Not all the definitions are stated as found in the literature, but in some cases they are reported in an equivalent formulation in order to adhere to our convention of providing non-entanglement definitions and to keep a uniform notation.
\par
The first category of works identified in Section \ref{sec:related_work} defines entanglement as a situation where a taut tether forms a bend at the point of contact with an obstacle \cite{rajan2016tether, teshnizi2014computing}.
\begin{definition}[Taut Tether Contact with Obstacle]
	\label{def:tether_contact_with_obstacle}
	Given two points $\x{a}, \x{r} \in \mathcal{X}_\text{free}$ and a taut tether configuration $\gamma$, the tether is to not entangled if $\gamma = l_{\x{a}, \x{r}}$, i.e., if the taut tether coincides with a straight line segment. 
\end{definition}
\par
Works in the second category consider multi-robot systems in obstacle-free environments and define entanglement as the creation of a bend in a tether due to the interaction with another robot's tether \cite{sinden1990tethered, hert1999motion}.
\begin{definition}[Taut Tether Contact with Other Tethers]
	\label{def:taut_tether_contact_with_other_tether}
	Let $\mathcal{I} = \{1, 2, \ldots, q\}$ represent a set of $q$ robots composing a multi-robot system in an environment where the tethers of the robots are the only obstacles.
	Given the robot $i\in\mathcal{I}$, which is located in $x_{\text{r}, i}$ and is connected through a taut tether $\gamma_i$ to its anchor point $x_{\text{a}, i}$, the tether $\gamma_i$ is not entangled if $\gamma_i = l_{x_{\text{a}, i}, x_{\text{r}, i}}$.
\end{definition}
A more general version of Definition \ref{def:taut_tether_contact_with_other_tether} is provided in \cite{cao2023neptune, cao2023path}, where entanglement between multiple slack tethers is defined.
\begin{definition}[Entanglement between Slack Tethers]
	\label{def:risk_of_entanglement}
	Let $\mathcal{I} = \{1, 2, \ldots, q\}$ represent a set of $q$ robots composing a multi-robot system in a 3D environment.
	Consider a robot $i\in\mathcal{I}$, its tether configuration $\gamma_i$, and a homotopy signature of its tether $\mathbbm{h}(\gamma_i)$ computed on a projection of the environment on a 2D plane.\footnote{More details on the projection procedure can be found in \cite{cao2023neptune, cao2023path}.}
	The tether configuration $\gamma_i$ is not entangled if its signature $\mathbbm{h}(\gamma_i)$ contains the letter $z_j$ corresponding to another robot $j\in\mathcal{I}\setminus \{i\}$ at most once.
\end{definition}
Intuitively, the definition corresponds to requiring that the tethers of two different robots do not cross each other more than once.
Otherwise, they could be at a configuration like the one shown in Figure \ref{figure:entanglement_example_multi_robot}, where the tethers are at risk of coming in contact if the two robots continue moving in their current directions.
\begin{figure}[t]
	\centering
	\includegraphics[scale=0.8]{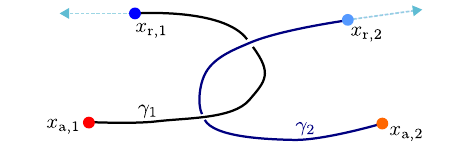}
	\caption{Example of entanglement with respect to Definition \ref{def:risk_of_entanglement}. Two robots and their respective tethers are shown in an obstacle-free 3D environment. If the robots continue moving along the dashed lines, the tethers will come into contact, restricting the motion capabilities of the two robots. The image is adapted from \cite{cao2023neptune}.}
	\label{figure:entanglement_example_multi_robot}
\end{figure}
\par
The third category of non-entanglement definitions considers tethers that form loops around obstacles. 
The first version of this definition considers only 2D scenarios and defines as entangled all configurations containing a loop (i.e., a self-intersection of the tether) around an obstacle \cite{shapovalov2020tangle}.
\begin{definition}[2D Tether Loop around Obstacle]
	\label{def:local_looping}
	In a 2D workspace, a tether configuration $\gamma$ is not entangled if, for any $s_1, s_2 \in [0,1], s_1 \leq s_2$ such that $\gamma(s_1) = \gamma(s_2)$, it holds that $\gamma_{[s_1, s_2]} \sim \gamma(s_1)$.
\end{definition}
Definition \ref{def:local_looping} requires that every loop in the tether is path-homotopic to a constant map.\footnote{In \cite{shapovalov2020tangle} this definition is stated as $\gamma(s) \!=\! \gamma(s') \! \iff \! s \!=\! s'$. However, in \cite{shapovalov2020tangle} a taut tether is assumed, which implies that loops can only happen around obstacles. The version of the definition reported here is a generalization of the definition to any type of tether configuration.}
The second version of this definition uses the same type of constraint but extends it to the whole tether configuration, considering only closed tether configurations, i.e., those where $\x{a} = \x{r}$. This corresponds to the situation where a robot has traveled in the environment and has returned to $\x{a}$. 
In this case, the requirement for non-entanglement is that the tether configuration can be homotopically deformed to the anchor point \cite{mccammon2017planning}. 
\begin{definition}[Closed Tether Homotopy to Constant Map]
	\label{def:closed_tether_homotopy_to_constant_map}
	A tether configuration $\gamma$ such that $\gamma(0) = \gamma(1) = x_\text{a}$ is not entangled if $\gamma \sim \x{a}$, i.e., if $\gamma$ is path-homotopic to $\x{a}$.
\end{definition}

\subsection{Proposed non-entanglement definitions}
\label{subsec:proposed_definitions}
The first two proposed non-entanglement definitions focus on some transformation that the tether must be able to achieve in order to be considered not entangled. In the first case, we want that any part of the tether can be made taut without encountering obstacles (i.e., that for any two points $x_1, x_2 \in \gamma$ we have $l_{x_1, x_2}\cap\interior\mathcal{O}=\emptyset$). We express this condition by requiring that the convex hull of the tether is obstacle-free.
\begin{definition}[Obstacle-free Convex Hull]
	\label{def:obstacle_free_conv_hull}
	A tether configuration $\gamma$ is not entangled if its convex hull does not intersect with any obstacle, i.e., $\conv(\gamma)\cap\interior\mathcal{O}=\emptyset$.
\end{definition}
In the second definition we require instead that the tether can be continuously retracted to the anchor point by shortening/rewinding the tether without encountering any obstacles in the path. 
\begin{definition}[Obstacle-free Linear Homotopy]
	\label{def:obstacle_free_linear_homotopy}
	A tether configuration $\gamma$ is not entangled if the linear homotopic map $H:[0,1]\times [0,1]\rightarrow\mathcal{X}$ defined by
	\begin{equation}
		H(s,t) = (1-t)\,\gamma(s) + t\,\x{a} 
		\label{eq:linear_homotopy}
	\end{equation}
	does not intersect with the interior of the obstacle region $\mathcal{O}$, i.e., if
	\begin{equation*}
		l_{\gamma(s), \x{a}}\cap\interior\mathcal{O}=\emptyset, \forall s\in[0,1].
	\end{equation*}
\end{definition}
An example of application of the Obstacle-free Convex Hull and Obstacle-free Linear Homotopy non-entanglement definitions (Definition \ref{def:obstacle_free_conv_hull} and \ref{def:obstacle_free_linear_homotopy}) is shown in Figure \ref{fig:conv_hull_linear_homotopy}.
\par
In practice, the Obstacle-free Convex Hull and Obstacle-free Linear Homotopy non-entanglement definitions (Definition \ref{def:obstacle_free_conv_hull} and \ref{def:obstacle_free_linear_homotopy}) result to be quite conservative in determining if a tether configuration is entangled or not. However, these definitions represent two simple criteria to define a set of non-entangled configurations that can be used as a starting point to define more complex and general non-entanglement definitions.
\begin{figure}[t]
	\centering
	\includegraphics[]{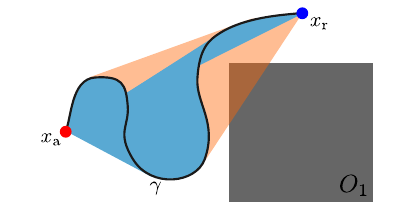}
	\caption{Example of the Obstacle-free Convex Hull and Obstacle-free Linear Homotopy non-entanglement definitions (Definition \ref{def:obstacle_free_conv_hull} and \ref{def:obstacle_free_linear_homotopy}) applied to a tether configuration $\bm{\gamma}$. The blue shaded region represents the set of points covered by the linear homotopic transformation $\bm{H}$ from $\bm{\gamma}$ and to $\bm{\x{a}}$. The blue shaded region does not intersect with the obstacle $\bm{O_1}$, so the configuration is not entangled with respect to the Obstacle-free Linear Homotopy definition (Definition \ref{def:obstacle_free_linear_homotopy}). However, the same configuration is entangled with respect to the Obstacle-free Convex Hull definition (Definition \ref{def:obstacle_free_conv_hull}). In fact, the light-orange shaded area, which corresponds to $\bm{\textrm{\textbf{conv}}(\gamma)}$, intersects with $\bm{O_1}$.}
	\label{fig:conv_hull_linear_homotopy}
\end{figure}
\par
A less conservative definition is introduced next. 
This definition determines if a given tether configuration is entangled or not on the base of the configurations that can be achieved by moving the terminal point of the tether, rather than evaluating the entanglement state of the tether solely based on its current configuration.
More specifically, we consider the set of configurations that can be obtained by moving $\x{r}$ along some path $\lambda$ belonging to a set $\Lambda_{\x{r}, p}$. 
The set $\Lambda_{\x{r}, P}$ consists of obstacle-free paths (i.e., such that $\lambda \cap \interior\mathcal{O} = \emptyset, \forall \lambda\in\Lambda_{\x{r}, P} $) having their initial point in the current location of the robot (i.e., $\lambda(0) = \x{r}$), and satisfying some property $P$. 
For example, the property $P$ can be defined as the length constraint $\mathrm{maxlen}:\len(\lambda) \leq d_\text{max}$, where $d_\text{max}>0$ is the maximum path length.
\begin{definition}[Path Homotopy to Safe Set]
	\label{def:homotopy_to_safe_set}
	Let $\Gamma_{x_\text{a}}^\text{safe}$ be a set of tether configurations that have their initial point at $\x{a}$, and which are considered to be safe, and let $\lambda\in\Lambda_{\x{r}, P}$ be a set of obstacle-free paths along which the robot can move from its current location and that satisfy property $P$. A tether configuration $\gamma$ from $\x{a}$ to $\x{r}$ is not entangled if there exists a path $\lambda \in \Lambda_{x_\text{r}, P}$ and a configuration $\bar{\gamma}\in\Gamma_{x_\text{a}}^\text{safe}$ such that $\gamma$ is relatively homotopic to $\bar{\gamma}$ along $\lambda$, i.e., if $\gamma\sim_{\x{a}, \lambda}\bar{\gamma}$.
\end{definition}
We remark that the equivalence $\sim_{\x{a}, \lambda}$ indicates that the two paths $\gamma$, $\bar{\gamma}$ must be relatively homotopic along the paths $\lambda$ and $x_\text{a}$, which means that the initial point of $\gamma$ and $\bar{\gamma}$ remains fixed at $\x{a}$ during the homotopic transformation.
\par
In the Path Homotopy to Safe Set non-entanglement definition, both the sets $\Gamma_{x_\text{a}}^\text{safe}$ and $\Lambda_{x_\text{r}, P}$ can be defined arbitrarily and can be adapted to the specific application, environment, and tethered robotic system being considered. 
This provides considerable flexibility to this entanglement definition, which can be made more or less conservative in detecting entanglement. For example, the sets  $\Gamma_{x_\text{a}}^\text{safe}$ and $\Lambda_{x_\text{r}, P}$ could be defined starting from the dynamic models of the tether and the robot, ensuring that safe tether configurations can actually be reached under the kynodynamical properties of the system. 
This, however, is at the same time also the main drawback of this definition since, in practice, the proper selection of the sets $\Gamma_{x_\text{a}}^\text{safe}$ and $\Lambda_{x_\text{r}, P}$ requires some knowledge of the properties of the robot and of the environment.
An example of application of the Path Homotopy to Safe Set non-entanglement definition is illustrated next.
\begin{figure}[t]
	\centering
	\includegraphics[scale=0.85]{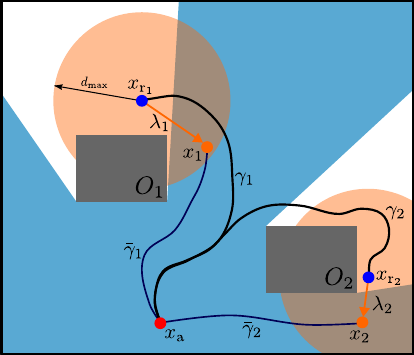}
	\caption{Example of the application of the Path Homotopy to Safe Set definition. The set of safe configurations $\bm{\Gamma_{\x{a}}}^\textrm{\textbf{safe}}$ is visualized as the set of all points that are reachable through at least one configuration that is not entangled with respect to the Obstacle-free Convex Hull definition (Definition \ref{def:obstacle_free_conv_hull}), and is represented as the blue shaded area. 
	The sets of paths $\bm{\Lambda_{x_{\textrm{\textbf{r}}_i}, \mathrm{maxlen}}}, i=1,2$ are defined as the sets of all straight line paths starting from $\bm{x_{r_i}}, i=1,2$ and having length less than or equal to $\bm{d_\textrm{\textbf{max}}}$. The sets $\bm{\Lambda_{x_{\textrm{\textbf{r}}_1},\mathrm{maxlen}}}$ and $\bm{\Lambda_{x_{\textrm{\textbf{r}}_2},\mathrm{maxlen}}}$ are visualized by the orange shaded areas.}
	\label{fig:path_to_safe_set}
\end{figure}
\begin{example}
	\label{example:safe_set_defintion}
	Let $\Gamma_{x_\text{a}}^\text{safe}$ be the set of all tether configurations having $x_a$ as initial point and satisfying the Obstacle-free Convex Hull non-entanglement definition (Definition \ref{def:obstacle_free_conv_hull}). 
	Also, for each $x_{\text{r}_i}$, let $\Lambda_{x_{\text{r}_i}, \mathrm{maxlen}}$ be the set of obstacle-free straight paths having their initial point in $x_{\text{r}_i}$ and satisfying the property $\mathrm{maxlen}:\len(\lambda) \leq d_\text{max}$. 
	Figure \ref{fig:path_to_safe_set} shows two example tether configurations for which we check the Path Homotopy to Safe Set definition given the sets $\Gamma_{x_\text{a}}^\text{safe}$ and $\Lambda_{x_\text{r}, \mathrm{maxlen}}$ just described. 
	The blue shaded area represents the set	of all the points that can be reached through a tether configuration that is not entangled according to Definition \ref{def:obstacle_free_conv_hull}.
	The tether configuration $\gamma_1$ is not entangled. In fact, the path $\lambda_1$ allows to reach the point $x_1$, for which there exists a safe non-entangled configuration according to the Obstacle-free Convex Hull definition (Definition \ref{def:obstacle_free_conv_hull}). 
	Most importantly, the non-entangled path $\bar{\gamma}_1$, which reaches $x_1$ from $\x{a}$ and which is not entangled according to the Obstacle-free Convex Hull definition (Definition \ref{def:obstacle_free_conv_hull}), is relatively homotopic to $\gamma_1$ along $\lambda_1$.
	On the contrary, $\gamma_2$ is entangled. In fact, despite the existence of the path $\lambda_2$ that goes from $x_{\text{r}_2}$ to $x_2$, the point $x_2$ is only reachable from $\x{a}$ through safe configurations that are not relatively homotopic to $\gamma_2$ along $\lambda_2$, e.g., the configuration $\bar{\gamma}_2$.
\end{example}
One more non-entanglement definition is introduced now. This definition is topology-based and determines the entanglement state of a tether configuration based on the tether location relative to the obstacles, and possibly to the other tethers, present in the environment.
This definition has some similarities with Definition \ref{def:risk_of_entanglement}, as it identifies as entangled tether configurations that go around an obstacle, as in the example of Figure \ref{figure:entanglement_example_multi_robot}. 
However, this non-entanglement definition improves Definition \ref{def:risk_of_entanglement} in two ways, namely, it considers both 2D and 3D environments with general types of obstacles, and it does not rely on the projection of the tether configurations on 2D planes.
In this definition we require that, if between two points $x_1$, $x_2\in\gamma$ there are no obstacles, i.e., $l_{x_1, x_2} \cap\interior\mathcal{O}=\emptyset$, then it must be possible to make the piece of tether between those two points taut without crossing any obstacles, i.e., the piece of tether between $x_1$, $x_2$ must be path-homotopic to $l_{x_1, x_2}$. 
\begin{definition}[Local Visibility Homotopy]
	\label{def:local_visibility_homotopy}
	A tether configuration $\gamma$ is not entangled if, for any pair of points $x_1 = \gamma(s_1)$, $x_2 = \gamma(s_2)$ such that  $l_{x_1, x_2}\cap\interior\,\mathcal{O}=\emptyset$, it holds that
	\begin{equation}
		\label{eq:local_visibility_homotopy}
		\gamma_{[s_1, s_2]} \sim l_{x_1, x_2}.
	\end{equation}
\end{definition}
\begin{figure}[t]
	\centering
	\includegraphics[scale=0.9]{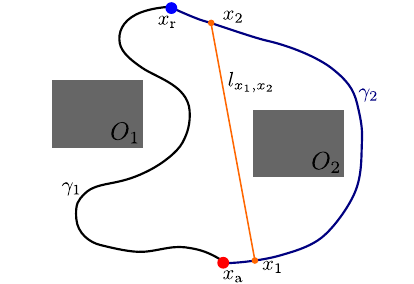}
	\caption{Example of the application of the Local Visibility Homotopy non-entanglement definition for two different tether configurations. Configuration $\bm{\gamma_{1}}$ is not entangled. On the contrary, $\bm{\gamma_{2}}$ is entangled. In fact, the part of the path $\bm{\gamma_{2}}$ between $\bm{x_{1}}$ and $\bm{x_{2}}$ is not homotopic to the straight line $\bm{l_{x_{1}, x_{2}}}$ between the points $\bm{x_{1}}$ and $\bm{x_{2}}$ due to the presence of obstacle $\bm{O_{2}}$.}
	\label{fig:local_visibility_homotopy}
\end{figure}
An example of application of the Local Visibility Homotopy definition is shown in Figure \ref{fig:local_visibility_homotopy}.
\par
We remark that for 3D multi-robot scenarios, the Local Visibility Homotopy non-entanglement definition (Definition \ref{def:local_visibility_homotopy}) does not always identify as entangled tether configurations that go around the tether of another robot. In fact, when the tether of another robot does not form a loop, there always exists a homotopy between a curve $\gamma_{[s_1, s_2]}$ and the corresponding straight-line segment $l_{x_1, x_2}$.
A possible solution for this is to require that the homotopic \Frechet distance between the curves $\gamma_{[s_1, s_2]}$ and $l_{x_1, x_2}$ is equal to or less than some value $\beta$, i.e., $\bar{\mathcal{F}}(\gamma_{[s_1, s_2]}, l_{x_1, x_2})\leq\beta$. A possible choice for the value of $\beta$ is $\beta = \len(\gamma_{[s_1, s_2]})$.

\subsection{Relaxation of the non-entanglement constraints}
\label{subsec:delta_relaxation}
In this section we introduce a relaxation of the non-entanglement definitions proposed up to this point, in order to make them less conservative in the detection of entanglement.
Given a tether configuration that is not entangled according to one of the non-entanglement definitions, we observe that a small variation in the configuration can lead to it being identified as entangled, as shown in the example depicted in Figure \ref{fig:local_visibility_homotopy_perturbation} for the Local Visibility Homotopy non-entanglement definition (Definition \ref{def:local_visibility_homotopy}). 
In many applications, limited violations of the non-entanglement constraints can be considered acceptable, as they do not immediately harm the mobility of a tethered robot. For this reason, it may be desirable to make the non-entanglement definitions less conservative, allowing a certain amount of violation of the constraint.
\par
To achieve this, we introduce a $\delta$-relaxed version of the non-entanglement definitions, which, given a tether configuration $\gamma$ that is not entangled according to the original definition, considers as not entangled any tether configuration $\gamma'$ that is path-homotopic to $\gamma$ and is $\delta$-close to $\gamma$.
The closeness between the two tether configurations $\gamma$ and $\gamma'$ is measured using the homotopic \Frechet distance.
\begin{definition}[$\delta$-Relaxed Non-Entanglement Definition $d$]
	\label{def:delta_relaxed_definition}
	A tether configuration $\gamma$ is considered to be not entangled if there exists a tether configuration $\gamma'$ such that:
	\begin{enumerate}[label=\roman*)]
		\item $\gamma'$ is not entangled according to Definition $d$;
		\item $\gamma\sim\gamma'$;
		\item $\bar{\mathcal{F}}(\gamma, \gamma') \leq \delta$;
	\end{enumerate}
	where $\delta\in[0,\infty]$ is the maximum value that the homotopic \Frechet distance between $\gamma$ and the target tether configuration $\gamma'$ can have.
\end{definition}
The value of $\delta$ can be arbitrarily chosen and determines the allowed amount of violation of the non-entanglement constraint of a given definition before a tether configuration is considered to be entangled. 
\begin{figure}[t]
	\centering
	\includegraphics[scale=0.95]{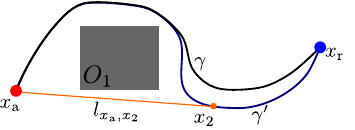}
	\caption{An example of the sensitivity of the Local Visibility Homotopy non-entanglement definition (Definition \ref{def:local_visibility_homotopy}) to variations in the tether configuration. Configuration $\bm{\gamma}$ is not entangled with respect to the definition. On the contrary, configuration $\bm{\gamma'}$ (which is identical to $\bm{\gamma}$ everywhere except for the variation above obstacle $\bm{O_1}$) is entangled, as the straight line $\bm{l_{\x{a}, x_2}}$ is not homotopic to $\bm{\gamma'_{\x{a}, x_2}}$.}
	\label{fig:local_visibility_homotopy_perturbation}
\end{figure}
We provide now an example of relaxation of a non-entanglement constraint. 
In particular, we discuss the $\infty$-relaxation of the Local Visibility Homotopy non-entanglement definition (Definition \ref{def:local_visibility_homotopy}). 
We remark that, by choosing $\delta=\infty$, we are effectively removing condition iii) from Definition \ref{def:delta_relaxed_definition}, resulting in the following definition.
\begin{definition}[Path Class-Relaxed Local Visibility Homotopy]
	\label{def:path_class_relaxed_local_visibility_homotopy}
	A tether configuration $\gamma$ is not entangled if there exists a tether configuration $\gamma'\in\Gamma_{\x{a}, \x{r}}$ such that:
	\begin{enumerate}[label=\roman*)]
		\item $\gamma'$ is not entangled according to Definition \ref{def:local_visibility_homotopy};
		\item $\gamma'\sim\gamma$.
	\end{enumerate}
\end{definition}
We choose to analyze this specific relaxation because, as will become more clear in Section \ref{sec:relationships} during the comparison of the definitions, the Path Class-Relaxed Local Visibility Homotopy definition generalizes well many of the other definitions. As we show in Section \ref{sec:relationships}, if a tether configuration is entangled according to this definition, then it is also entangled according to many of the other definitions.
\par
The Path Class-Relaxed Local Visibility Homotopy non-entanglement definition extends the non-entanglement property from a path $\gamma$ to the path homotopy class $[\gamma]$, i.e., if a path $\gamma$ is not entangled according to the Local Visibility Homotopy definition (Definition \ref{def:local_visibility_homotopy}), then any path $\gamma'$ such that $[\gamma'] = [\gamma]$ is also not entangled.
This means that, in the  Path Class-Relaxed Local Visibility Homotopy non-entanglement definition the entanglement state of a tether does not depend on the specific tether configuration, but on the path homotopy class in which it lies. For example, in a scenario such as the one depicted in Figure \ref{fig:local_visibility_homotopy_perturbation} both configurations would be considered not entangled since they belong to the same path homotopy class.

\section{Characterization of the free workspace under the non-entanglement definitions}
\label{sec:properties}
In this section, we characterize the part of the free workspace that is reachable by a tethered robot under the different entanglement definitions.
In fact, given a non-entanglement definition, in general only a subset of the workspace can be reached through a non-entangled tether configuration. This effectively limits the part of the free workspace in which a tethered robot can move without getting its tether entangled. 
Knowing this restriction is useful during the development of entanglement-free motion planning algorithms for tethered robots.
The proofs of the results presented in this section are reported in Appendices \ref{appendix:proofs_properties}.
\bigbreak
Given a free workspace $\mathcal{X}_\text{free}$, an anchor point location $x_\text{a}$, and a non-entanglement definition $d$, we find that in general only a subset of the points of $\mathcal{X}_\text{free}$ can be reached through a non-entangled tether configuration. We call this set the \emph{non-entangled free workspace}, and define it as
\begin{align*}
	\mathcal{N}_{\x{a}, d} =& \{x : \exists \gamma\in\Gamma_{x_\text{a}, x_\text{r}}, \\
	& \text{s.t. } \gamma \text{ is not entangled under Definition $d$} \}.
\end{align*}
\par
We start the characterization of the non-entangled free workspace from the non-entanglement definitions found in the literature, i.e., Definitions \ref{def:tether_contact_with_obstacle} -- \ref{def:closed_tether_homotopy_to_constant_map}. 
We note that for Definition \ref{def:risk_of_entanglement} the non-entangled free workspace is not computed, since the set $\mathcal{N}_{\x{a}, \ref{def:risk_of_entanglement}}$ depends on the tether configurations of the other robots, and not just on $\mathcal{X}_\text{free}$ and $x_\text{a}$. We also observe that $\mathcal{N}_{\x{a}, \ref{def:closed_tether_homotopy_to_constant_map}} = \{x_\text{a}\}$, since in Definition \ref{def:closed_tether_homotopy_to_constant_map} only closed tether configurations are considered, i.e., where $x_\text{r} = x_\text{a}$.
For Definitions \ref{def:tether_contact_with_obstacle} and \ref{def:taut_tether_contact_with_other_tether} the sets $\mathcal{N}_{\x{a}, \ref{def:tether_contact_with_obstacle}}$ and $\mathcal{N}_{\x{a}, \ref{def:taut_tether_contact_with_other_tether}}$ are straightforward to find.
\begin{proposition}
	\label{prop:reachable_set_straight_tether}
	The sets $\mathcal{N}_{\x{a}, \ref{def:tether_contact_with_obstacle}}$ and $\mathcal{N}_{\x{a}, \ref{def:taut_tether_contact_with_other_tether}}$ coincide, and they are equal to the set of points that can be reached from $\x{a}$ through a straight line segment that does not intersect with any obstacle, i.e., $\mathcal{N}_{\x{a}, \ref{def:tether_contact_with_obstacle}} = \mathcal{N}_{\x{a}, \ref{def:taut_tether_contact_with_other_tether}} = \{ x: \exists l_{x, \x{a}}, l_{x, \x{a}} \cap\interior\mathcal{O}=\emptyset\}$.
\end{proposition}
For Definition \ref{def:local_looping} we show instead that all points in $\mathcal{X}_\text{free}$ are part of the non-entangled free workspace.
\begin{proposition}
	\label{prop:reachable_set_non_looping_tether}
	The non-entangled free workspace for the 2D Tether Loop around Obstacle non-entanglement definition (Definition \ref{def:local_looping}) is given by $\mathcal{N}_{\x{a}, \ref{def:local_looping}} = \mathcal{X}_\text{free}$.
\end{proposition}
\par
We move now to the study of the set $\mathcal{N}_{\x{a}, d}$ for the proposed non-entanglement definitions.
We start by observing that the sets $\mathcal{N}_{\x{a}, \ref{def:obstacle_free_conv_hull}}$ and $\mathcal{N}_{\x{a}, \ref{def:obstacle_free_linear_homotopy}}$ coincide, and that they are the same as the non-entangled workspace of Definitions \ref{def:tether_contact_with_obstacle} and \ref{def:taut_tether_contact_with_other_tether} that was characterized in Proposition \ref{prop:reachable_set_straight_tether}.
\begin{proposition}
	\label{prop:reachable_set_conv_hull_linear_homotopy}
	The sets $\mathcal{N}_{\x{a}, \ref{def:obstacle_free_conv_hull}}$ and $\mathcal{N}_{\x{a}, \ref{def:obstacle_free_linear_homotopy}}$ coincide, and they are equal to the set of points that can be reached from $\x{a}$ through a straight line segment that does not intersect with any obstacle, i.e., $\mathcal{N}_{\x{a}, \ref{def:obstacle_free_conv_hull}} = \mathcal{N}_{\x{a}, \ref{def:obstacle_free_linear_homotopy}} = \{ x: \exists l_{x, \x{a}}, l_{x, \x{a}} \cap\interior\mathcal{O}=\emptyset\}$.
\end{proposition}
This result, despite its simplicity, highlights well the strictness of the Obstacle-free Convex Hull and the Obstacle-free Linear Homotopy non-entanglement definitions (Definitions \ref{def:obstacle_free_conv_hull}  and \ref{def:obstacle_free_linear_homotopy}). In fact, Proposition \ref{prop:reachable_set_conv_hull_linear_homotopy} shows how these definitions do not allow the tether to go around obstacles, since in both definitions each point of the tether must always be in the line of sight of the anchor point. 
\par 
Moving onto the Path Homotopy to Safe Set definition (Definition \ref{def:homotopy_to_safe_set}), we observe that the set $\mathcal{N}_{\x{a}, \ref{def:homotopy_to_safe_set}}$ can be computed by extending the analysis described in Example \ref{example:safe_set_defintion} for the two points $x_{\text{r}_1}, x_{\text{r}_2}$ to all the points in $\mathcal{X}_\text{free}$. To check if a point $x_\text{r}\in\mathcal{X}_\text{free}$ belongs to $\mathcal{N}_{\x{a}, \ref{def:homotopy_to_safe_set}}$ it is necessary to determine if there exists a tether configuration $\gamma$ for which it is possible to find a path $\lambda\in\Lambda_{x_\text{r}, P}$ and a safe configuration $\bar{\gamma}\in\Gamma_{x_\text{a}}^\text{safe}$ such that $\gamma$ is relatively homotopic to $\bar{\gamma}$ along $\lambda$, i.e., $\mathcal{N}_{\x{a}, \ref{def:homotopy_to_safe_set}} = \left\{x_\text{r} : \exists \bar{\gamma}\in\Gamma_{x_\text{a}}^\text{safe}, \lambda\in\Lambda_{x_\text{r}, P} \text{ s.t. } (\bar{\gamma} \diamond \lambda^\text{reverse})(1) = x_\text{r} \right\}$.
In general, $\mathcal{N}_{\x{a}, \ref{def:homotopy_to_safe_set}}$ cannot be characterized more explicitely, as it depends on the specific choice of the sets $\Gamma_{x_\text{a}}^\text{safe}$ and $\Lambda_{x_\text{r}, P}$.
\begin{figure}[t]
	\centering
	\includegraphics[width=0.8\columnwidth]{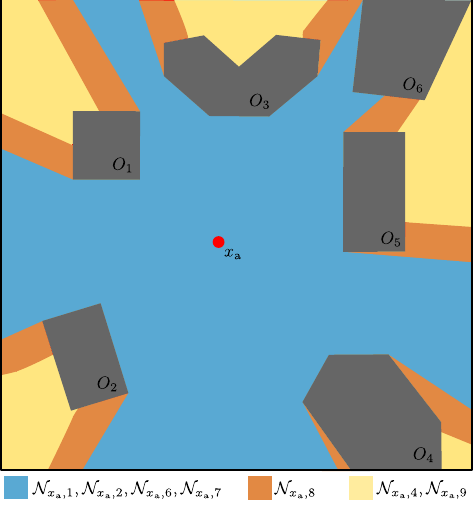}
	\caption{
		Comparison of the non-entangled free workspace $\bm{\mathcal{N}_{\x{a}, d}}$ for the Definitions \ref{def:tether_contact_with_obstacle},  \ref{def:taut_tether_contact_with_other_tether}, \ref{def:local_looping}, and \ref{def:obstacle_free_conv_hull}--\ref{def:local_visibility_homotopy}.
		All the sets are computed starting from the anchor point $\bm{\x{a}}$ shown in the middle of the image. 
		The sets $\bm{\mathcal{N}_{\x{a}, \ref{def:tether_contact_with_obstacle}}}$, $\bm{\mathcal{N}_{\x{a}, \ref{def:taut_tether_contact_with_other_tether}}}$, $\bm{\mathcal{N}_{\x{a}, \ref{def:obstacle_free_conv_hull}}}$, and $\bm{\mathcal{N}_{\x{a}, \ref{def:obstacle_free_linear_homotopy}}}$ correspond to all the points that are in an obstacle-free line of sight with $\bm{\x{a}}$.
		To compute the set $\bm{\mathcal{N}_{\x{a}, \ref{def:homotopy_to_safe_set}}}$, the set $\bm{\Gamma_{\x{a}}}^\textrm{\textbf{safe}}$ is defined as the set of tether configurations that are not entangled with respect to the Obstacle-Free Linear Homotopy definition (Definition \ref{def:obstacle_free_linear_homotopy}), and the set $\bm{\Lambda_{x_{\textrm{\textbf{r}}}, p}}$ is defined as the set of all paths such that $\textrm{\textbf{len}}\bm{(\lambda) \leq d_\textrm{\textbf{max}}}$. 
		The sets $\bm{\mathcal{N}_{\x{a}, \ref{def:local_looping}}}$ and $\bm{\mathcal{N}_{\x{a}, \ref{def:local_visibility_homotopy}}}$ cover the whole $\bm{\mathcal{X}_\textrm{\textbf{free}}}$.
	}
	\label{fig:reachable_set_example}
\end{figure}
\par
Lastly, for the Local Visibility Homotopy non-entanglement definition (Definition \ref{def:local_visibility_homotopy}) we show that, given an anchor point $\x{a}$, every point $x\in\mathcal{X}_\text{free}$ is reachable through at least one non-entangled tether configuration.
\begin{proposition}
	\label{prop:local_homotopy_reachable_set}
	The non-entangled free workspace for the Local Visibility Homotopy non-entanglement definition is given by $\mathcal{N}_{\x{a}, \ref{def:local_visibility_homotopy}} = \mathcal{X}_\text{free}$.
\end{proposition}
\par
To conclude this section, we compare the non-entangled free workspace for the different definitions. 
In the comparison below, for the Path Homotopy to Safe Set (Definition \ref{def:homotopy_to_safe_set}) we define the set $\Gamma_{x_\text{a}}^\text{safe}$ as the set of tether configurations starting from $x_\text{a}$ that are not entangled with respect to the Obstacle-Free Linear Homotopy definition (Definition \ref{def:obstacle_free_linear_homotopy}), and the set $\Lambda_{x_\text{r}, \mathrm{maxlen}}$ is defined as the set of all paths starting from $x_\text{r}$ having length less than or equal to some $d_\text{max}$.
We obtain then the following relations
\begin{equation*}
	\mathcal{N}_{\x{a}, \ref{def:tether_contact_with_obstacle}} =
	\mathcal{N}_{\x{a}, \ref{def:taut_tether_contact_with_other_tether}} =
	\mathcal{N}_{\x{a}, \ref{def:obstacle_free_conv_hull}} = \mathcal{N}_{\x{a}, \ref{def:obstacle_free_linear_homotopy}} \subseteq \mathcal{N}_{\x{a}, \ref{def:homotopy_to_safe_set}} \subseteq \mathcal{N}_{\x{a}, \ref{def:local_looping}} =
	\mathcal{N}_{\x{a}, \ref{def:local_visibility_homotopy}}.
\end{equation*}
We also observe that, in case of an obstacle-free workspace all the sets coincide. 
An example of the non-entangled workspace for the different definitions is shown in Figure \ref{fig:reachable_set_example}.

\section{Relationships between the definitions}
\label{sec:relationships}
In Section \ref{sec:definitions} a broad set of non-entanglement definitions has been introduced and discussed. There exist a number of relationships between the different definitions.
We are specifically interested in determining if a definition is a special case of another one, i.e., if being entangled with respect to one definition implies being entangled also with respect to another one.
It is worth noting that some of the definitions can only be applied under specific environment and tether conditions.
These conditions are the following:
\begin{enumerate}[left=0.5cm, label={}]
	\item \textbf{C1.} Taut tether configuration;
	\item \textbf{C2.} Multi-robot system;
	\item \textbf{C3.} Obstacle-free environment;
	\item \textbf{C4.} 2D environment;
	\item \textbf{C5.} Closed tether configuration (coinciding endpoints).
\end{enumerate}
The definitions that require specific types of tether configurations and environments are summarized below:
\renewcommand{\arraystretch}{1.5}
\begin{itemize}[left=0.5cm, label={}]
	\item \textbf{Definition \ref{def:tether_contact_with_obstacle}:} C1
	\item \textbf{Definition \ref{def:taut_tether_contact_with_other_tether}:} C1, C2, C3
	\item \textbf{Definition \ref{def:risk_of_entanglement}:} C2
	\item \textbf{Definition \ref{def:local_looping}:} C4
	\item \textbf{Definition \ref{def:closed_tether_homotopy_to_constant_map}:} C5
\end{itemize}
\renewcommand{\arraystretch}{1}
Definitions not listed here can handle generic tether configurations and environments. 
\par
Table \ref{tab:definition comparison} summarizes all the relationships between the different definitions.
\begin{table*}[!ht]
	\centering
	\caption{Comparison of the Entanglement Definitions.}
	\begin{tabular}{|c|>{\centering}m{0.5cm}>{\centering}m{0.55cm}>{\centering}m{0.55cm}>{\centering}m{0.55cm}>{\centering}m{0.55cm}>{\centering}m{0.5cm}>{\centering}m{0.5cm}>{\centering}m{0.5cm}>{\centering}m{0.5cm}>{\centering}m{0.7cm}>{\centering}m{0.9cm}c|}
		\hline
		& \rotatebox{75}{\begin{Tabular}[1]{@{}c@{}} \ref{def:tether_contact_with_obstacle}. Taut Tether Contact \\ with Obstacle (C1) \end{Tabular}} 
		& \rotatebox{75}{\begin{Tabular}[1]{@{}c@{}} \ref{def:taut_tether_contact_with_other_tether}. Taut Tether Contact \\ with Other Tethers \\ (C1, C2, C3) \end{Tabular}} 
		& \rotatebox{75}{\begin{Tabular}[1]{@{}c@{}} \ref{def:risk_of_entanglement}. Entanglement between \\ Slack Tethers (C2)  \end{Tabular}} 
		& \rotatebox{75}{\begin{Tabular}[1]{@{}c@{}} \ref{def:local_looping}. 2D Tether Loop \\ around Obstacle (C4)  \end{Tabular}} 
		& \rotatebox{75}{\begin{Tabular}[1]{@{}c@{}} \ref{def:closed_tether_homotopy_to_constant_map}. Closed Tether Homotopy \\ to Constant Map (C5)  \end{Tabular}} 
		& \rotatebox{75}{\begin{Tabular}[1]{@{}c@{}} \ref{def:obstacle_free_conv_hull}. Obstacle-Free \\ Convex Hull \end{Tabular}} 
		& \rotatebox{75}{\begin{Tabular}[1]{@{}c@{}} \ref{def:obstacle_free_linear_homotopy}. Obstacle-Free \\ Linear Homotopy \end{Tabular}} 
		& \rotatebox{75}{\begin{Tabular}[1]{@{}c@{}} \ref{def:homotopy_to_safe_set}. Path Homotopy \\ to Safe Set \end{Tabular}} 
		& \rotatebox{75}{\begin{Tabular}[1]{@{}c@{}} \ref{def:local_visibility_homotopy}. Local Visibility \\ Homotopy \end{Tabular}} 
		& \rotatebox{75}{
			\begin{Tabular}[1]{@{}c@{}} \ref{def:delta_relaxed_definition}. $\delta$-Relaxed \\ Non-Entanglement \\ Definition $d$ \end{Tabular}}
		& \rotatebox{75}{\begin{Tabular}[1]{@{}c@{}} \ref{def:path_class_relaxed_local_visibility_homotopy}. Path Class Robustified \\ Local Visibility Homotopy \end{Tabular}} & \\
		\hline
		\begin{Tabular}[1]{@{}c@{}} \ref{def:tether_contact_with_obstacle}. Taut Tether Contact \TstrutSmall \\ with Obstacle (C1)\BstrutSmall \end{Tabular} & X & & & X & & X & X & & X & & X & \\
		\begin{Tabular}[1]{@{}c@{}} \ref{def:taut_tether_contact_with_other_tether}. Taut Tether Contact \TstrutSmall \\ with Other Tethers \\ (C1, C2, C3)\BstrutSmall \end{Tabular} & X & X & X & X & X & X & X & & X & & X & \\
		\begin{Tabular}[1]{@{}c@{}} \ref{def:risk_of_entanglement}. Entanglement between \TstrutSmall \\ Slack Tethers (C2)\BstrutSmall \end{Tabular} & & & X & & & & & & & & & \\
		\begin{Tabular}[1]{@{}c@{}} \ref{def:local_looping}. 2D Tether Loop\TstrutSmall \\ around Obstacle (C4)\BstrutSmall \end{Tabular} & & & & X & X & & & & & & X\textsuperscript{(C5)} & \\
		\begin{Tabular}[1]{@{}c@{}} \ref{def:closed_tether_homotopy_to_constant_map}. Closed Tether Homotopy \TstrutSmall \\ to Constant Map (C5)\BstrutSmall \end{Tabular} & & & & X & X & & & & & & X & \\
		\begin{Tabular}[1]{@{}c@{}} \ref{def:obstacle_free_conv_hull}. Obstacle-Free\TstrutSmall \\ Convex Hull\BstrutSmall \end{Tabular} & X & & & X & X & X & X & & X & & X & \\
		\begin{Tabular}[1]{@{}c@{}} \ref{def:obstacle_free_linear_homotopy}. Obstacle-Free\TstrutSmall \\ Linear Homotopy\BstrutSmall \end{Tabular} & & & & X & X & & X & & & & X & \\
		\begin{Tabular}[1]{@{}c@{}} \ref{def:homotopy_to_safe_set}.Path Homotopy \TstrutSmall \\ to Safe Set \BstrutSmall \end{Tabular} & & & & & & & & X & & & & \\
		\begin{Tabular}[1]{@{}c@{}} \ref{def:local_visibility_homotopy}. Local Visibility\TstrutSmall \\ Homotopy\BstrutSmall \end{Tabular} & & & & X & X & & & & X & & X & \\
		\begin{Tabular}[1]{@{}c@{}} \ref{def:delta_relaxed_definition}. $\delta$-Relaxed \TstrutSmall \\ \BstrutSmall Non-Entanglement \TstrutSmall \\ Definition $d$ \BstrutSmall \end{Tabular} & & & & & & & & & & X & & \\
		\begin{Tabular}[1]{@{}c@{}} \ref{def:path_class_relaxed_local_visibility_homotopy}. Path Class Robustified \TstrutSmall \\ Local Visibility Homotopy\BstrutSmall \end{Tabular} & & & & & & & & & & & X & \\
		\hline
	\end{tabular}
	\label{tab:definition comparison}
\end{table*}
In the table, a cross indicates that being not entangled with respect to the definition on a given row implies being not entangled also with respect to the definition in the corresponding column (and, conversely, that being entangled with respect to the definition in a given column implies being entangled also with respect to the definition in the corresponding row). 
When a relationship exists between two definitions, we assume that the conditions required by the two definitions, which are indicated in Table \ref{tab:definition comparison} next to the number of each definition, are simultaneously satisfied.
Some of the crosses have additional conditions indicated by a superscript, which means that the relationship between the two definitions is true if those additional conditions hold.
The proofs of the relationships listed in Table \ref{tab:definition comparison} are reported in Appendix \ref{appendix:proofs_comparison}.
\par
From the analysis of Table \ref{tab:definition comparison} it is possible to gain an intuition of which non-entanglement definitions are more strict and which are less so in identifying a tether configuration as non-entangled. 
The definitions that imply non-entanglement also according to many other definitions (i.e., those whose corresponding row contain many crosses) are typically more strict. 
For instance, Definition \ref{def:tether_contact_with_obstacle} considers a tether configuration as non-entangled only if it coincides with a straight line segment, which is also considered to be a non-entangled configuration by most of the other definitions. 
A similar argument also holds for Definitions \ref{def:taut_tether_contact_with_other_tether} and \ref{def:obstacle_free_conv_hull}. 
On the contrary, definitions according to which non-entanglement is also implied by many others are usually more general. This is the case of Definitions \ref{def:local_looping} and \ref{def:closed_tether_homotopy_to_constant_map}, which generalize most of the other definitions in their specific cases of application (respectively, 2D environments and closed tether configurations). Definition \ref{def:path_class_relaxed_local_visibility_homotopy} also turns out to generalize most of the other definitions.
\par
We remark that the way in which the Path to Safe Homotopy definition (Definition \ref{def:homotopy_to_safe_set}) is related to other definitions depends on the specific choice of $\Gamma_{x_\text{a}}^\text{safe}$ and $\Lambda_{x_\text{r}, P}$. For this reason, no relationship has been marked in the table for Definition \ref{def:homotopy_to_safe_set}. The same holds for the $\delta$-Relaxed Non-Entanglement definition (Definition \ref{def:delta_relaxed_definition}), as the relations of Definition \ref{def:delta_relaxed_definition} depend on the specific choice of $\delta$ and $d$.

\section{Empirical validation}
\label{sec:validation}
The non-entanglement definitions  presented in this article are intended to be applied in tethered robot systems to characterize the entanglement state of the tether. However, the definitions are not straightforward to validate since, as already discussed, there is not a well-established and generally accepted definition of entanglement to compare them with. 
For this reason, we have opted for a qualitative validation of the definitions by experts in the field of tethered robotics. A total of 12 experts from the field of tethered robotics have been asked to evaluate a set of test scenarios.
Each scenario is composed by a set of obstacles, an anchor point $x_\text{a}$, a robot location $x_\text{r}$, and a tether configuration $\gamma$. The considered scenarios include both 2D and 3D environments, single-robot and multi-robot systems, and both loose and taut tether configurations. Three examples of test scenarios are shown in Figure \ref{fig:example_scenarios}. The full list of validation scenarios can be found in the supplemental material to this article.
\begin{figure*}[!t]
	\centering
	\hfill
	\subfloat[\label{fig:example_scenarios_1}][]{
		\includegraphics[width=0.3\textwidth]{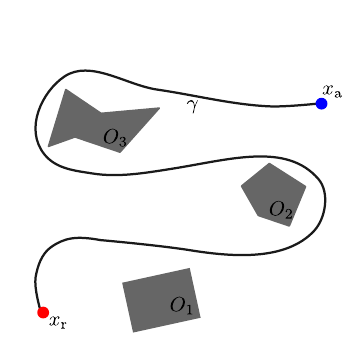}
	}
	\hfill
	\subfloat[\label{fig:example_scenarios_2}][]{
		\includegraphics[width=0.3\textwidth]{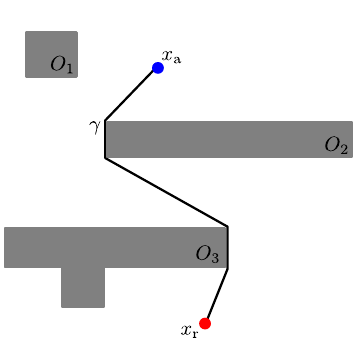}
	}
	\hfill
	\subfloat[\label{fig:example_scenarios_3}][]{
		\includegraphics[width=0.3\textwidth]{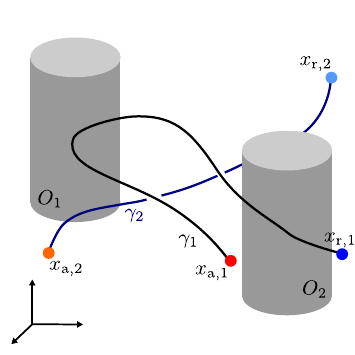}
	}
	\hfill\phantom{.}
	\caption{Three examples of validation scenarios used for the empirical validation of the binary non-entanglement definitions. In the figures, gray regions $\bm{O_i}$ represent the obstacles, the blue dot $\bm{\x{a}}$ indicates the anchor point location, the red dot $\bm{\x{r}}$ indicates the robot location, and the black curve $\bm{\gamma}$ represents the tether configuration under analysis. Scenarios (a) and (b) are 2D and single-robot. Scenario (c) is 3D and multi-robot.}
	\label{fig:example_scenarios}
\end{figure*}
\par
All the non-entanglement definitions discussed in this article are applied to each scenario to determine the entanglement state of the tether.\footnote{For Definition \ref{def:homotopy_to_safe_set} the set $\Gamma_{x_\text{a}}^\text{safe}$ is defined as the set of all the tether configurations that are not entangled according to Definition \ref{def:obstacle_free_conv_hull}, while $\Lambda_{x_\text{r}, \mathrm{maxlen}}$ is defined as the set of all obstacle-free straight paths starting from $x_\text{r}$ and having length less than or equal to some $d_\text{max}$, as done e.g. in Example \ref{example:safe_set_defintion}. In 3D multi-robot scenarios, Definitions \ref{def:local_visibility_homotopy} and \ref{def:path_class_relaxed_local_visibility_homotopy} have been considered under the requirement that $\bar{\mathcal{F}}(\gamma_{[s_1, s_2]}, l_{x_1, x_2})\leq\len(\gamma_{[s_1, s_2]})$, which was discussed at the end of Section \ref{subsec:proposed_definitions}.}
The experts performed the same operation by indicating, for each scenario, if they would consider the tether configuration to be entangled or not.
In addition to indicating if a tether configuration is entangled or not, the experts had the possibility to indicate the extent on the entanglement in that given scenario. The four possible answers that the experts could select are `N' (not entangled), `W' (weakly entangled), `E' (entangled), and `S' (strongly entangled).
The non-entanglement definitions can produce instead three different outcomes: `N' (not entangled), `E' (entangled), and `--' (definition not applicable, e.g., when the conditions required by a certain definition are not satisfied in the scenario under analysis).
The results of the validation process are shown in Table \ref{tab:definitions_validation}.
\renewcommand{\arraystretch}{1}
\setlength\heavyrulewidth{0.2ex}
\begin{table*}[!ht]
	\centering
	\caption{Validation of the entanglement definitions. N: not entangled; W: weakly entangled; E: entangled; S: strongly entangled; --: not applicable.}
	\label{tab:definitions_validation}
	\vspace{-0.3cm}
	\setlength\heavyrulewidth{0.25ex}

\definecolor{back_N}{rgb}{1,  .78,  .808}
\definecolor{text_N}{rgb}{.612,  0,  .024}
\definecolor{back_W}{rgb}{.886,  .937,  .855}
\definecolor{text_W}{rgb}{.216,  .337,  .137}
\definecolor{back_E}{rgb}{.776,  .878,  .706}
\definecolor{text_E}{rgb}{.216,  .337,  .137}
\definecolor{back_S}{rgb}{.663,  .816,  .557}
\definecolor{text_S}{rgb}{.216,  .337,  .137}
\definecolor{back_NA}{rgb}{1,  1,  1}		
\definecolor{text_NA}{rgb}{.612,  .341,  0}

\newcommand{\answerN}{\cellcolor{back_N}\textcolor{text_N}{N}}
\newcommand{\answerW}{\cellcolor{back_W}\textcolor{text_W}{W}}
\newcommand{\answerE}{\cellcolor{back_E}\textcolor{text_E}{E}}
\newcommand{\answerS}{\cellcolor{back_S}\textcolor{text_S}{S}}
\newcommand{\answerNA}{\cellcolor{back_NA}\textcolor{text_NA}{--}}

\begin{tabular}{ccccccccccccccccccccccccc}
	\toprule
	\textbf{Scenario} & \multicolumn{12}{c}{\textbf{Expert}} & & \multicolumn{10}{c}{\textbf{Definition}} & \\
	& \textbf{1} & \textbf{2} & \textbf{3} & \textbf{4} & \textbf{5} & \textbf{6} & \textbf{7} & \textbf{8} & \textbf{9} & \textbf{10} & \textbf{11} & \textbf{12} & & 
	\ref{def:tether_contact_with_obstacle} &
	\ref{def:taut_tether_contact_with_other_tether} & 
	\ref{def:risk_of_entanglement} & 
	\ref{def:local_looping} &
	\ref{def:closed_tether_homotopy_to_constant_map} &
	\ref{def:obstacle_free_conv_hull} &
	\ref{def:obstacle_free_linear_homotopy} & 
	\ref{def:homotopy_to_safe_set} & 
	\ref{def:local_visibility_homotopy} & 
	\ref{def:path_class_relaxed_local_visibility_homotopy} &  \\
	\midrule
	\textbf{A1} & \answerN & \answerN & \answerN & \answerN & \answerN & \answerN & \answerN & \answerN & \answerN & \answerN & \answerN & \answerN &       & \answerNA & \answerNA & \answerNA & \answerN & \answerNA & \answerE & \answerE & \answerN & \answerE & \answerN &  \\
	\textbf{A2} & \answerN & \answerN & \answerN & \answerN & \answerN & \answerN & \answerW & \answerN & \answerN & \answerN & \answerN & \answerN &       & \answerNA & \answerNA & \answerNA & \answerN & \answerNA & \answerE & \answerE & \answerN & \answerN & \answerN &  \\
	\textbf{A3} & \answerN & \answerN & \answerN & \answerN & \answerE & \answerN & \answerN & \answerN & \answerN & \answerE & \answerN & \answerN &       & \answerNA & \answerNA & \answerNA & \answerN & \answerNA & \answerE & \answerE & \answerN & \answerN & \answerN &  \\
	\textbf{A4} & \answerN & \answerN & \answerN & \answerN & \answerN & \answerN & \answerN & \answerN & \answerN & \answerN & \answerN & \answerN &       & \answerNA & \answerNA & \answerNA & \answerN & \answerNA & \answerE & \answerE & \answerN & \answerN & \answerN &  \\
	\textbf{B1} & \answerN & \answerN & \answerN & \answerN & \answerN & \answerN & \answerE & \answerN & \answerN & \answerN & \answerN & \answerN &       & \answerNA & \answerNA & \answerNA & \answerN & \answerNA & \answerE & \answerE & \answerN & \answerN & \answerN &  \\
	\textbf{B2} & \answerN & \answerN & \answerN & \answerN & \answerN & \answerN & \answerW & \answerN & \answerN & \answerN & \answerN & \answerN &       & \answerNA & \answerNA & \answerNA & \answerN & \answerNA & \answerE & \answerE & \answerN & \answerN & \answerN &  \\
	\textbf{B3} & \answerN & \answerN & \answerN & \answerN & \answerN & \answerN & \answerN & \answerN & \answerN & \answerN & \answerN & \answerN &       & \answerNA & \answerNA & \answerNA & \answerN & \answerNA & \answerE & \answerE & \answerE & \answerE & \answerE &  \\
	\textbf{B4} & \answerN & \answerN & \answerN & \answerN & \answerN & \answerN & \answerN & \answerN & \answerN & \answerN & \answerN & \answerW &       & \answerNA & \answerNA & \answerNA & \answerN & \answerNA & \answerE & \answerE & \answerN & \answerE & \answerN &  \\
	\textbf{B5} & \answerN & \answerN & \answerN & \answerN & \answerN & \answerN & \answerN & \answerN & \answerN & \answerN & \answerN & \answerN &       & \answerNA & \answerNA & \answerNA & \answerN & \answerNA & \answerE & \answerE & \answerN & \answerE & \answerN &  \\
	\textbf{B6} & \answerN & \answerN & \answerN & \answerN & \answerN & \answerN & \answerN & \answerN & \answerN & \answerN & \answerN & \answerW &       & \answerNA & \answerNA & \answerNA & \answerN & \answerNA & \answerE & \answerE & \answerN & \answerN & \answerN &  \\
	\textbf{C1} & \answerN & \answerN & \answerW & \answerN & \answerN & \answerN & \answerE & \answerW & \answerN & \answerE & \answerW & \answerW &       & \answerNA & \answerNA & \answerNA & \answerN & \answerNA & \answerE & \answerE & \answerE & \answerE & \answerE &  \\
	\textbf{C2} & \answerE & \answerW & \answerE & \answerW & \answerS & \answerW & \answerS & \answerE & \answerE & \answerS & \answerE & \answerS &       & \answerNA & \answerNA & \answerNA & \answerE & \answerNA & \answerE & \answerE & \answerE & \answerE & \answerE &  \\
	\textbf{D1} & \answerW & \answerN & \answerS & \answerN & \answerE & \answerN & \answerS & \answerW & \answerW & \answerE & \answerE & \answerW &       & \answerE & \answerNA & \answerNA & \answerN & \answerNA & \answerE & \answerE & \answerE & \answerN & \answerN &  \\
	\textbf{D2} & \answerN & \answerN & \answerN & \answerN & \answerN & \answerN & \answerW & \answerN & \answerN & \answerW & \answerN & \answerN &       & \answerNA & \answerNA & \answerNA & \answerN & \answerNA & \answerE & \answerE & \answerE & \answerN & \answerN &  \\
	\textbf{D3} & \answerE & \answerN & \answerW & \answerN & \answerW & \answerN & \answerS & \answerW & \answerW & \answerE & \answerE & \answerW &       & \answerE & \answerNA & \answerNA & \answerN & \answerNA & \answerE & \answerE & \answerE & \answerN & \answerN &  \\
	\textbf{D4} & \answerE & \answerN & \answerW & \answerN & \answerW & \answerN & \answerS & \answerW & \answerW & \answerE & \answerE & \answerW &       & \answerE & \answerNA & \answerNA & \answerN & \answerNA & \answerE & \answerE & \answerE & \answerN & \answerN &  \\
	\textbf{E1} & \answerN & \answerN & \answerN & \answerN & \answerN & \answerN & \answerW & \answerN & \answerN & \answerN & \answerN & \answerN &       & \answerNA & \answerNA & \answerNA & \answerN & \answerNA & \answerE & \answerE & \answerN & \answerN & \answerN &  \\
	\textbf{E2} & \answerE & \answerN & \answerW & \answerN & \answerW & \answerN & \answerS & \answerW & \answerW & \answerE & \answerE & \answerW &       & \answerE & \answerNA & \answerNA & \answerN & \answerNA & \answerE & \answerE & \answerE & \answerE & \answerE &  \\
	\textbf{E3} & \answerN & \answerN & \answerN & \answerN & \answerN & \answerN & \answerW & \answerN & \answerN & \answerN & \answerN & \answerN &       & \answerNA & \answerNA & \answerNA & \answerN & \answerNA & \answerE & \answerE & \answerE & \answerE & \answerE &  \\
	\textbf{E4} & \answerN & \answerN & \answerN & \answerN & \answerE & \answerN & \answerN & \answerN & \answerN & \answerE & \answerN & \answerN &       & \answerNA & \answerNA & \answerNA & \answerN & \answerNA & \answerE & \answerE & \answerN & \answerN & \answerN &  \\
	\textbf{F1} & \answerN & \answerN & \answerN & \answerN & \answerN & \answerN & \answerN & \answerN & \answerN & \answerN & \answerN & \answerN &       & \answerNA & \answerNA & \answerNA & \answerN & \answerNA & \answerE & \answerE & \answerN & \answerN & \answerN &  \\
	\textbf{F2} & \answerN & \answerN & \answerN & \answerN & \answerN & \answerN & \answerN & \answerN & \answerN & \answerN & \answerN & \answerN &       & \answerNA & \answerNA & \answerNA & \answerN & \answerNA & \answerE & \answerE & \answerN & \answerE & \answerN &  \\
	\textbf{F3} & \answerN & \answerN & \answerN & \answerN & \answerN & \answerN & \answerN & \answerN & \answerN & \answerN & \answerN & \answerN &       & \answerNA & \answerNA & \answerNA & \answerN & \answerNA & \answerE & \answerE & \answerN & \answerN & \answerN &  \\
	\textbf{F4} & \answerN & \answerN & \answerN & \answerN & \answerN & \answerN & \answerW & \answerN & \answerN & \answerW & \answerN & \answerN &       & \answerNA & \answerNA & \answerNA & \answerN & \answerNA & \answerE & \answerE & \answerE & \answerE & \answerE &  \\
	\textbf{G1} & \answerN & \answerN & \answerN & \answerN & \answerN & \answerN & \answerN & \answerN & \answerN & \answerE & \answerW & \answerW &       & \answerNA & \answerNA & \answerNA & \answerN & \answerNA & \answerE & \answerE & \answerE & \answerN & \answerN &  \\
	\textbf{G2} & \answerN & \answerN & \answerN & \answerN & \answerE & \answerN & \answerN & \answerN & \answerN & \answerE & \answerW & \answerW &       & \answerNA & \answerNA & \answerNA & \answerN & \answerNA & \answerE & \answerE & \answerE & \answerN & \answerN &  \\
	\textbf{H1} & \answerE & \answerN & \answerE & \answerN & \answerW & \answerN & \answerS & \answerW & \answerW & \answerE & \answerE & \answerW &       & \answerE & \answerNA & \answerNA & \answerN & \answerNA & \answerE & \answerE & \answerE & \answerN & \answerN &  \\
	\textbf{H2} & \answerN & \answerN & \answerN & \answerN & \answerN & \answerN & \answerE & \answerN & \answerN & \answerN & \answerN & \answerN &       & \answerNA & \answerNA & \answerNA & \answerN & \answerNA & \answerE & \answerE & \answerE & \answerE & \answerN &  \\
	\textbf{I1} & \answerN & \answerN & \answerN & \answerN & \answerN & \answerN & \answerW & \answerN & \answerN & \answerN & \answerN & \answerN &       & \answerNA & \answerNA & \answerNA & \answerN & \answerNA & \answerE & \answerE & \answerE & \answerE & \answerE &  \\
	\textbf{I2} & \answerN & \answerN & \answerN & \answerN & \answerN & \answerN & \answerN & \answerN & \answerN & \answerN & \answerN & \answerN &       & \answerNA & \answerNA & \answerNA & \answerN & \answerNA & \answerE & \answerE & \answerE & \answerE & \answerE &  \\
	\textbf{I3} & \answerE & \answerN & \answerN & \answerN & \answerW & \answerN & \answerS & \answerN & \answerN & \answerE & \answerN & \answerE &       & \answerNA & \answerNA & \answerNA & \answerN & \answerNA & \answerE & \answerE & \answerE & \answerE & \answerE &  \\
	\textbf{I4} & \answerN & \answerN & \answerN & \answerN & \answerN & \answerN & \answerN & \answerN & \answerN & \answerN & \answerN & \answerN &       & \answerNA & \answerNA & \answerNA & \answerN & \answerNA & \answerN & \answerN & \answerN & \answerN & \answerN &  \\
	\textbf{I5} & \answerN & \answerN & \answerN & \answerN & \answerN & \answerN & \answerN & \answerN & \answerN & \answerN & \answerN & \answerN &       & \answerNA & \answerNA & \answerNA & \answerN & \answerNA & \answerE & \answerE & \answerN & \answerN & \answerN &  \\
	\textbf{I6} & \answerN & \answerN & \answerW & \answerN & \answerW & \answerN & \answerE & \answerW & \answerN & \answerW & \answerW & \answerW &       & \answerNA & \answerNA & \answerNA & \answerN & \answerNA & \answerE & \answerE & \answerN & \answerE & \answerE &  \\
	\textbf{L1} & \answerN & \answerN & \answerN & \answerN & \answerN & \answerN & \answerN & \answerN & \answerN & \answerN & \answerN & \answerN &       & \answerNA & \answerNA & \answerN & \answerN & \answerNA & \answerE & \answerE & \answerE & \answerN & \answerN &  \\
	\textbf{L2} & \answerN & \answerN & \answerN & \answerN & \answerN & \answerN & \answerN & \answerN & \answerN & \answerN & \answerN & \answerN &       & \answerNA & \answerNA & \answerN & \answerN & \answerNA & \answerE & \answerN & \answerN & \answerN & \answerN &  \\
	\textbf{L3} & \answerN & \answerN & \answerN & \answerN & \answerN & \answerN & \answerN & \answerN & \answerN & \answerN & \answerN & \answerN &       & \answerNA & \answerNA & \answerN & \answerN & \answerNA & \answerE & \answerE & \answerN & \answerN & \answerN &  \\
	\textbf{L4} & \answerN & \answerN & \answerN & \answerN & \answerE & \answerN & \answerN & \answerN & \answerN & \answerE & \answerN & \answerW &       & \answerNA & \answerNA & \answerE & \answerN & \answerNA & \answerE & \answerE & \answerE & \answerE & \answerE &  \\
	\textbf{M1} & \answerW & \answerN & \answerN & \answerN & \answerW & \answerN & \answerN & \answerW & \answerN & \answerN & \answerN & \answerW &       & \answerNA & \answerNA & \answerN & \answerNA & \answerNA & \answerE & \answerE & \answerN & \answerE & \answerN &  \\
	\textbf{M2} & \answerN & \answerN & \answerN & \answerN & \answerE & \answerN & \answerN & \answerN & \answerN & \answerN & \answerN & \answerN &       & \answerNA & \answerNA & \answerN & \answerNA & \answerNA & \answerN & \answerN & \answerN & \answerN & \answerN &  \\
	\textbf{M3} & \answerN & \answerN & \answerE & \answerN & \answerS & \answerW & \answerE & \answerE & \answerW & \answerW & \answerE & \answerE &       & \answerNA & \answerNA & \answerE & \answerNA & \answerNA & \answerE & \answerE & \answerE & \answerE & \answerE &  \\
	\textbf{M4} & \answerE & \answerN & \answerW & \answerN & \answerS & \answerW & \answerE & \answerE & \answerW & \answerW & \answerE & \answerE &       & \answerNA & \answerNA & \answerE & \answerNA & \answerNA & \answerE & \answerE & \answerE & \answerE & \answerE &  \\
	\textbf{M5} & \answerE & \answerN & \answerN & \answerN & \answerS & \answerN & \answerN & \answerW & \answerN & \answerN & \answerN & \answerW &       & \answerNA & \answerNA & \answerE & \answerNA & \answerNA & \answerN & \answerN & \answerN & \answerN & \answerN &  \\
	\textbf{M6} & \answerE & \answerN & \answerW & \answerW & \answerS & \answerW & \answerE & \answerE & \answerW & \answerE & \answerE & \answerE &       & \answerNA & \answerNA & \answerE & \answerNA & \answerNA & \answerE & \answerE & \answerE & \answerE & \answerE &  \\
	\textbf{N1} & \answerN & \answerW & \answerN & \answerN & \answerN & \answerN & \answerE & \answerN & \answerN & \answerN & \answerN & \answerN &       & \answerNA & \answerNA & \answerNA & \answerN & \answerNA & \answerE & \answerE & \answerE & \answerN & \answerN &  \\
	\textbf{N2} & \answerN & \answerW & \answerN & \answerN & \answerN & \answerN & \answerS & \answerW & \answerN & \answerE & \answerW & \answerW &       & \answerNA & \answerNA & \answerNA & \answerN & \answerNA & \answerE & \answerE & \answerE & \answerN & \answerN &  \\
	\textbf{N3} & \answerN & \answerN & \answerN & \answerN & \answerW & \answerW & \answerE & \answerW & \answerN & \answerE & \answerW & \answerW &       & \answerE & \answerNA & \answerNA & \answerN & \answerNA & \answerE & \answerE & \answerE & \answerN & \answerN &  \\
	\textbf{N4} & \answerE & \answerN & \answerW & \answerN & \answerW & \answerW & \answerS & \answerW & \answerW & \answerE & \answerW & \answerW &       & \answerE & \answerNA & \answerNA & \answerN & \answerNA & \answerE & \answerE & \answerE & \answerE & \answerE &  \\
	\textbf{N5} & \answerN & \answerE & \answerN & \answerN & \answerN & \answerN & \answerE & \answerN & \answerN & \answerN & \answerN & \answerW &       & \answerNA & \answerNA & \answerNA & \answerN & \answerNA & \answerE & \answerE & \answerE & \answerE & \answerE &  \\
	\textbf{N6} & \answerN & \answerE & \answerW & \answerN & \answerE & \answerE & \answerS & \answerE & \answerE & \answerS & \answerE & \answerW &       & \answerE & \answerNA & \answerNA & \answerN & \answerNA & \answerE & \answerE & \answerE & \answerN & \answerN &  \\
	\textbf{N7} & \answerN & \answerN & \answerW & \answerN & \answerW & \answerW & \answerS & \answerE & \answerW & \answerW & \answerE & \answerN &       & \answerE & \answerNA & \answerNA & \answerN & \answerNA & \answerE & \answerE & \answerN & \answerN & \answerN &  \\
	\textbf{N8} & \answerE & \answerS & \answerE & \answerW & \answerS & \answerE & \answerS & \answerE & \answerN & \answerS & \answerS & \answerE &       & \answerNA & \answerNA & \answerNA & \answerN & \answerNA & \answerE & \answerE & \answerE & \answerE & \answerE &  \\
	\textbf{N9} & \answerN & \answerN & \answerN & \answerN & \answerN & \answerN & \answerN & \answerN & \answerN & \answerN & \answerW & \answerN &       & \answerNA & \answerNA & \answerNA & \answerN & \answerNA & \answerN & \answerN & \answerE & \answerN & \answerN &  \\
	\textbf{N10} & \answerN & \answerN & \answerN & \answerN & \answerN & \answerW & \answerW & \answerW & \answerN & \answerN & \answerN & \answerW &       & \answerNA & \answerNA & \answerNA & \answerN & \answerNA & \answerE & \answerE & \answerE & \answerE & \answerN &  \\
	\textbf{O1} & \answerN & \answerN & \answerN & \answerN & \answerN & \answerN & \answerN & \answerN & \answerN & \answerN & \answerN & \answerN &       & \answerNA & \answerNA & \answerNA & \answerN & \answerN & \answerN & \answerN & \answerN & \answerN & \answerN &  \\
	\textbf{O2} & \answerE & \answerW & \answerW & \answerS & \answerS & \answerE & \answerS & \answerE & \answerN & \answerS & \answerE & \answerE &       & \answerNA & \answerNA & \answerNA & \answerE & \answerE & \answerE & \answerE & \answerE & \answerE & \answerE &  \\
	\textbf{O3} & \answerN & \answerN & \answerN & \answerN & \answerN & \answerN & \answerN & \answerN & \answerN & \answerN & \answerN & \answerN &       & \answerNA & \answerNA & \answerNA & \answerN & \answerN & \answerE & \answerE & \answerE & \answerE & \answerN &  \\
	\textbf{O4} & \answerE & \answerW & \answerN & \answerS & \answerE & \answerS & \answerE & \answerE & \answerN & \answerS & \answerE & \answerE &       & \answerNA & \answerNA & \answerNA & \answerE & \answerE & \answerE & \answerE & \answerE & \answerE & \answerE &  \\
	\textbf{O5} & \answerE & \answerN & \answerN & \answerS & \answerS & \answerS & \answerW & \answerE & \answerN & \answerS & \answerE & \answerS &       & \answerNA & \answerNA & \answerNA & \answerE & \answerE & \answerE & \answerE & \answerE & \answerE & \answerE &  \\
	\textbf{O6} & \answerN & \answerN & \answerN & \answerN & \answerN & \answerN & \answerN & \answerN & \answerN & \answerN & \answerE & \answerN &       & \answerNA & \answerNA & \answerNA & \answerN & \answerN & \answerE & \answerE & \answerE & \answerN & \answerN &  \\
	\textbf{P1} & \answerN & \answerN & \answerN & \answerN & \answerN & \answerN & \answerN & \answerN & \answerN & \answerN & \answerN & \answerN &       & \answerNA & \answerNA & \answerE & \answerNA & \answerNA & \answerE & \answerE & \answerN & \answerN & \answerN &  \\
	\textbf{P2} & \answerE & \answerW & \answerW & \answerE & \answerE & \answerE & \answerS & \answerW & \answerE & \answerS & \answerE & \answerE &       & \answerNA & \answerNA & \answerE & \answerNA & \answerNA & \answerE & \answerE & \answerE & \answerE & \answerE &  \\
	\textbf{P3} & \answerS & \answerE & \answerE & \answerS & \answerS & \answerE & \answerS & \answerW & \answerE & \answerS & \answerS & \answerS &       & \answerNA & \answerNA & \answerE & \answerNA & \answerNA & \answerE & \answerE & \answerE & \answerE & \answerE &  \\
	\textbf{P4} & \answerE & \answerW & \answerN & \answerW & \answerE & \answerE & \answerS & \answerW & \answerN & \answerS & \answerE & \answerE &       & \answerNA & \answerNA & \answerN & \answerNA & \answerNA & \answerE & \answerE & \answerE & \answerE & \answerE &  \\
	\textbf{P5} & \answerE & \answerW & \answerN & \answerE & \answerE & \answerE & \answerE & \answerW & \answerN & \answerS & \answerE & \answerE &       & \answerNA & \answerNA & \answerE & \answerNA & \answerNA & \answerE & \answerE & \answerE & \answerE & \answerE &  \\
	\textbf{P6} & \answerS & \answerE & \answerW & \answerS & \answerE & \answerE & \answerS & \answerW & \answerN & \answerS & \answerS & \answerW &       & \answerNA & \answerNA & \answerE & \answerNA & \answerNA & \answerE & \answerE & \answerE & \answerE & \answerE &  \\
	\textbf{P7} & \answerN & \answerE & \answerN & \answerE & \answerW & \answerE & \answerN & \answerN & \answerN & \answerE & \answerE & \answerE &       & \answerNA & \answerNA & \answerE & \answerNA & \answerNA & \answerE & \answerE & \answerE & \answerE & \answerN &  \\
	\textbf{P8} & \answerE & \answerS & \answerS & \answerS & \answerS & \answerS & \answerS & \answerE & \answerW & \answerS & \answerS & \answerS &       & \answerNA & \answerNA & \answerE & \answerNA & \answerNA & \answerE & \answerE & \answerE & \answerE & \answerE &  \\
	\bottomrule
\end{tabular}

\end{table*}
\par
By observing the left-hand part of the table, where the opinions of the experts are reported, it is easy to note that the evaluations of the test scenarios by the experts are often very different from each other. In fact, some of the experts tend to classify tether configurations as entangled more often than others, which indicates that their own definition of entanglement is more conservative than that of others. On the other hand, some experts only indicate a few scenarios as entangled (e.g., experts 2 and 4).
The right-hand side of the table shows the results of the application of the non-entanglement definitions to the test scenarios. 
Here it is possible to observe how the definitions differ in the evaluation of the scenarios, with some definitions being more conservative than others.
By comparing the two parts of the table, one can observe which definitions result to be closer to the opinions of the experts. Definitions \ref{def:tether_contact_with_obstacle}, \ref{def:risk_of_entanglement}, and \ref{def:closed_tether_homotopy_to_constant_map}, for instance, coincide almost always with the average opinion from the experts. However, these definitions need specific conditions to be satisfied, and can only be applied to a limited number of the test scenarios.
On the contrary, the new definitions proposed in this work (Definitions \ref{def:obstacle_free_conv_hull}--\ref{def:local_visibility_homotopy} and \ref{def:path_class_relaxed_local_visibility_homotopy}) can be applied to all the test scenarios.
It can also be noted, as already observed in Section \ref{sec:properties}, that Definitions \ref{def:obstacle_free_conv_hull} and \ref{def:obstacle_free_linear_homotopy} are quite conservative in evaluating entanglement, but also that they provide a good definition of a safe set for Definition \ref{def:homotopy_to_safe_set}, which results instead to be closer to the opinion of the experts. 
Definitions \ref{def:local_visibility_homotopy} and \ref{def:path_class_relaxed_local_visibility_homotopy} often coincide with the average opinion of the experts, with Definition \ref{def:path_class_relaxed_local_visibility_homotopy} being the closest one.

\section{Conclusions}
\label{sec:conclusions}
We have considered the problem of defining tether entanglement for tethered robots, in order to determine if a tether configuration is entangled or not. 
We have reviewed the entanglement definitions available in the literature and proposed several new entanglement definitions. 
All these definitions can be used to evaluate the entanglement state of a tether configuration. 
We have discussed the properties of the different definitions, highlighting their individual strengths and weaknesses, and analyzed the relationships between them.
In particular, the comparison of the definitions shows how some of the newly proposed definitions generalize many of the definitions existing in the literature, resulting in more comprehensive definitions of entanglement. 
\par
The main direction for future work regards the integration of the proposed definitions in motion planning algorithms for tethered robots, with the goal of obtaining safer and more robust trajectories. This entails the development of general motion planning algorithms for tethered robots that can make use of different entanglement definitions, that are robust to uncertainties in the localization of the tether and the obstacles, and that are able to find disentangling paths in case a robots has an entangled tether.
\par
A second important open research direction is the development of \emph{continuous} entanglement definitions based on the measure of a level of entanglement. By using this type of entanglement definitions, a set of tether configurations can be ordered relatively to each other depending on their level of entanglement. 
This type of definition can find application in entanglement-aware motion planning algorithms for tethered robots that focus on keeping the tether at a minimum level of entanglement, that is, that focus on optimizing the risk/safety level of the tethered robot.
For example, in the motion planning problem depicted in Figure \ref{fig:example_motion_planning}, a continuous entanglement definition can provide a way to rank the possible motion paths depending on the level of entanglement of the tether configurations resulting from the motion of the robot along those paths.
\par
Other open issues include the integration of self-knotting in the entanglement definitions, which is an unwanted condition that none of the entanglement models available in the literature captures effectively. This condition occurs in 3D when a tether passes through a loop or `eyelet' created by itself, and can lead to critical entanglement scenarios.

\appendices
\section{Proofs of Sections IV and V}
\label{appendix:proofs_definitions}
In this appendix we provide some technical results that have been used in Sections \ref{sec:preliminaries} and \ref{sec:definitions}.

\begin{lemma}[\hspace{1sp}\cite{hatcher2005algebraic}, p.25]
	\label{lemma:path_homotopy_in_convex_set}
	Given a convex subset of $\mathcal{Y}\subseteq\mathbb{R}^n$, all paths in $\mathcal{Y}$ with given endpoints $x_1$ and $x_2$ are path-homotopic to each other.
\end{lemma}

\begin{lemma}[Path homotopy in loops]
	\label{lemma:null_homotopic_loops}
	Let $\mathcal{Y}$ be a path-connected space, and $\gamma:[0,1]\rightarrow\mathcal{Y}$ be a loop, i.e., a path such that $\gamma(0) = \gamma(1) = x_0$, that is path-homotopic to its base point $x_0$. 
	Then for any two paths $\gamma_1$, $\gamma_2$ such that $\gamma_1\diamond\gamma_2^\text{reverse} = \gamma$, it holds that $\gamma_1 \sim \gamma_2$.
\end{lemma}

\begin{proof}
	Let $x_1 = \gamma_1(0) = \gamma_2(0)$ and $x_2 = \gamma_1(1) = \gamma_2(1)$. 
	Since $\mathcal{Y}$ is path-connected, $\gamma$ is path-homotopic to any of its points, which implies that $\gamma\sim x_1$ and $\gamma\sim x_2$.
	By using elementary properties of how path homotopy is preserved under path concatenation \cite[Theorem~7.11]{lee2010introduction} we have $\gamma_1 \sim \gamma_1 \diamond x_2 \sim \gamma_1 \diamond (\gamma_2^\text{reverse} \diamond \gamma_2) \sim (\gamma_1 \diamond \gamma_2^\text{reverse}) \diamond \gamma_2 \sim x_1 \diamond \gamma_2 \sim \gamma_2$, i.e., $\gamma_1 \sim \gamma_2$.
\end{proof}

\begin{lemma}[Existence of shortest path]
	\label{lemma:existence_shortest_path}
	Given any two points $x_1, x_2 \in \mathcal{X}_\text{free}$ there exists a shortest admissible path $\gamma^*$ between those two points, where $\gamma^* \in \arg\min_{\gamma\in\Gamma_{x_1, x_2}}[\len(\gamma)]$.
\end{lemma}

\begin{proof}
	The free space $\mathcal{X}_\text{free}$ is a \emph{boundedly compact metric space}, i.e., all closed bounded sets in it are compact \cite[Definition~1.6.7]{burago2001course}, since it is a closed subset of $\mathbb{R}^n$. Therefore, by \cite[Corollary~2.5.20]{burago2001course} there exists a shortest path between any two path-connected points in $\mathcal{X}_\text{free}$.
\end{proof}

Next, we formalize the \emph{homotopy signature}, which is a topological invariant that uniquely identifies the homotopy equivalence class of any path $\gamma\in\Gamma_{x_\text{a}, x_\text{r}}$.
Given a path $\gamma$, its signature is indicated as $\mathbbm{h}(\gamma)$. 
A signature is a \emph{word} generated as the free product of a finite set of letters \cite{bhattacharya2018path}.
All paths belonging to the same homotopy class have the same signature. Several approaches are available for the identification of homotopy classes through the use of signatures, both in two and three dimensions \cite{bhattacharya2010search, bhattacharya2015persistent, tovar2010sensor}. 
We largely base our definition of homotopy signature on \cite{bhattacharya2018path}.
In the following we refer to a continuous mapping $\beta:I\rightarrow\mathcal{X}_\text{free}$, where $I\subset\mathbb{R}$ is an interval, as a \emph{continuous curve}.

\begin{definition*}[Transversality]
	In $\mathbb{R}^2$, a path $\gamma$ and a curve $\zeta$ are said to be transversal if at every point of intersection between them they have distinct tangents \cite{guillemin1974differential}.
\end{definition*}

\begin{definition*}[Complete invariant \cite{faticoni2016modules}]
	\label{def:complete_invariant}
	A function $\mathbbm{h}$ from $\Gamma_{x_1, x_2}$ to the set $\Gamma_{x_1, x_2}/\sim$ is called a complete homotopy invariant if
	\begin{equation*}
		\mathbbm{h}(\gamma_1) = \mathbbm{h}(\gamma_2) \iff \gamma_1 \sim \gamma_2.
	\end{equation*}
\end{definition*}

\begin{construction}[Signature of a path]
	\label{construction:signature}
	Given a 2-dimensional manifold $\mathcal{X}_\text{free}$, let $\zeta_1, \zeta_2,\ldots, \zeta_p$ be continuous curves, called \emph{representative curves}, such that $\partial \zeta_i \subseteq \partial \mathcal{X}_\text{free}$.\footnote{$\partial$ indicates the \emph{boundary of a manifold}, which in case of a curve corresponds to its endpoints. For example, in case of a line segment $\zeta$ the boundary $\partial \zeta$ corresponds to the two endpoints of $\zeta$, while if $\zeta$ is a ray the boundary corresponds to the initial point of $\zeta$.}
	Then, for any two fixed points $x_\text{a}$, $x_\text{r}$, given a path $\gamma$ connecting $x_\text{a}$ and $x_\text{r}$ that is in general position (transverse) with respect to the $\zeta_i$, and that crosses the $\zeta_i$ a finite number of times, it is possible to construct a word by following the path from the start to the end and inserting in the word the letter $z_i$ or $z^{-1}_i$ whenever the path intersects the curve $\zeta_i$ with positive or negative orientation respectively. 
	By deleting any string of the type $z_i z^{-1}_i$ and $z^{-1}_i z_i$ we obtain a \emph{reduced word}.
\end{construction}

\begin{proposition}[Signature is a complete invariant]
	\label{prop:signature_is_invariant}
	Reduced words constructed as described in Construction \ref{construction:signature} are complete homotopy invariants for paths in $\mathcal{X}_\text{free}$ joining two given points $x_\text{a}$ and $x_\text{r}$ if the following conditions hold:
	\begin{enumerate}[label=\roman*)]
		\item $\zeta_i \cap \zeta_j = \emptyset, \forall i \neq j$;
		\item $\mathcal{X}_\text{free} \setminus \bigcup\limits_{i=1}^p \zeta_i$ is path-connected and simply connected;
		\item $\pi_1(\mathcal{X}_\text{free} \setminus \bigcup\limits_{i=1, i \neq j}^p \zeta_i) \cong \mathbb{Z}, \forall j \in \{1, \ldots, p\}$,
	\end{enumerate}
	where $\cong$ indicates a group isomorphism.
\end{proposition}
\begin{proof}
	The proof is provided on page 143 of \cite{bhattacharya2018path}.
\end{proof}

We call the homotopy invariant obtained from Construction~\ref{construction:signature} and Proposition \ref{prop:signature_is_invariant} \emph{homotopy signature}, and we indicate it with $\mathbbm{h}(\cdot)$.

\section{Proofs of Section VI}
\label{appendix:proofs_properties}

\begin{proof}[Proof of Proposition \ref{prop:reachable_set_straight_tether}]
	The proof of this result is trivial. In fact, Definitions \ref{def:tether_contact_with_obstacle} and \ref{def:taut_tether_contact_with_other_tether} are only applicable to taut tether configurations, and require the tether to be a straight line segment for it to not be considered entangled.
\end{proof}

\begin{proof}[Proof of Proposition \ref{prop:reachable_set_non_looping_tether}]
	In Lemma \ref{lemma:existence_shortest_path} we have established that, given the topological properties of $\mathcal{X}_\text{free}$, there exists a shortest path between any given pair of points $x_1$, $x_2$. Since a shortest path never contains a loop, that means that for any point $x\in\mathcal{X}_\text{free}$ there exists a shortest path $\gamma^*_{x_\text{a}, x}$ that does not contain a loop, and that therefore is not entangled according to Definition \ref{def:local_looping}. Thus, $\mathcal{N}_{\x{a}, \ref{def:local_looping}} = \mathcal{X}_\text{free}$.
\end{proof}

\begin{proof}[Proof of Proposition \ref{prop:reachable_set_conv_hull_linear_homotopy}]
	Proposition \ref{prop:reachable_set_conv_hull_linear_homotopy} states that $\mathcal{N}_{\x{a}, \ref{def:obstacle_free_conv_hull}} = \mathcal{N}_{\x{a}, \ref{def:obstacle_free_linear_homotopy}} = \{ x: l_{x, \x{a}} \cap\interior\mathcal{O}=\emptyset\}$.
	Given a point $x\in\mathcal{X}_\text{free}$ such that $l_{\x{a}, x} \cap\interior\mathcal{O}=\emptyset$, the path $\gamma=l_{\x{a}, x}$ from $\x{a}$ to $x$ is a non-entangled tether configuration with respect to both Definition \ref{def:obstacle_free_conv_hull} and Definition \ref{def:obstacle_free_linear_homotopy}, and therefore $x\in\mathcal{N}_{\x{a}, \ref{def:obstacle_free_conv_hull}}$ and $x\in\mathcal{N}_{\x{a}, \ref{def:obstacle_free_linear_homotopy}}$.
	Conversely, if there exists a tether configuration from $x_\text{a}$ to $x$ with $x\in\mathcal{N}_{\x{a}, \ref{def:obstacle_free_conv_hull}}$ or $x\in\mathcal{N}_{\x{a}, \ref{def:obstacle_free_linear_homotopy}}$, then $l_{\x{a}, x} \cap\interior\mathcal{O}=\emptyset$, since $l_{\x{a}, x}$ is part of $\conv(\gamma)$ in the former case, and of the linear path homotopy defined in \eqref{eq:linear_homotopy} in the latter case.
	It follows that $\mathcal{N}_{\x{a}, \ref{def:obstacle_free_conv_hull}} = \mathcal{N}_{\x{a}, \ref{def:obstacle_free_linear_homotopy}} = \{x \in \mathcal{X}_\text{free}: l_{\x{a}, x}\cap\interior\mathcal{O}=\emptyset\}$.
\end{proof}

\begin{lemma}
	\label{lemma:local_homotopy}
	Given a pair of fixed points $x_1, x_2$, the shortest tether configuration $\gamma^*_{x_1, x_2}$ always satisfies the Local Visibility Homotopy definition (Definition \ref{def:local_visibility_homotopy}).
\end{lemma}

\begin{proof}[Proof (by contradiction)]
	Let $\gamma^*$ be the shortest path between two points in $\mathcal{X}_\text{free}$ and suppose that $\gamma^*$ is entangled with respect to the Local Visibility Homotopy definition (Definition \ref{def:local_visibility_homotopy}). 
	This means that there exist two points $x_1 = \gamma^*(s_1), x_2 = \gamma^*(s_2)$, such that $l_{x_1, x_2} \not\sim \gamma^*_{[s_1, s_2]}$. Since $l_{x_1, x_2}$ is the shortest path between the two points $x_1$ and $x_2$, and $l_{x_1, x_2}$ and $\gamma^*_{[s_1, s_2]}$ are not homotopic, which means that they cannot coincide, we have $\len(l_{x_1, x_2}) < \len(\gamma^*_{[s_1, s_2]})$. 
	Therefore, the path $\gamma'$ that is obtained by replacing $\gamma^*_{[s_1, s_2]}$ by $l_{x_1, x_2}$ in the path $\gamma^*$ is shorter than $\gamma^*$. However, this is a contradiction since $\gamma^*$ was assumed to be the shortest path between the two points.
\end{proof}

\begin{proof}[Proof of Proposition \ref{prop:local_homotopy_reachable_set}]
	In Lemma \ref{lemma:existence_shortest_path} we have established that, since $\mathcal{X}_\text{free}$ is path connected, there always exists a shortest path between any pair of points in $\mathcal{X}_\text{free}$, and therefore also between a point $x\in\mathcal{X}_\text{free}$ and $x_\text{a}$. 
	Thus, we obtain from Lemma \ref{lemma:local_homotopy} that between any pair of points in $\mathcal{X}_\text{free}$ there exists a tether configuration that is not entangled according to Definition \ref{def:local_visibility_homotopy}.
\end{proof}

\section{Proofs of the comparisons of the non-entanglement definitions (section VII)}
\label{appendix:proofs_comparison}

In this appendix the relationships between the non-entanglement definitions that were introduced in Section \ref{sec:relationships} are proved. 
The proofs are given in the form `Definition $d_1$ implies Definition $d_2$' which means `if a tether configuration $\gamma$ is not entangled according to Definition $d_1$, then it is also not entangled with respect to Definition $d_2$'.
The relationships proven in this sections are visualized in Figure \ref{fig:definitions_relationships_graph}. It is worth noting that only the relationships indicated by a black arrow are proven, while those indicated by the light-grey arrows, which correspond to relationships that can be proven by concatenating other relationships, are not.

\begin{figure*}[t]
	\centering
	\includegraphics[width=1\textwidth]{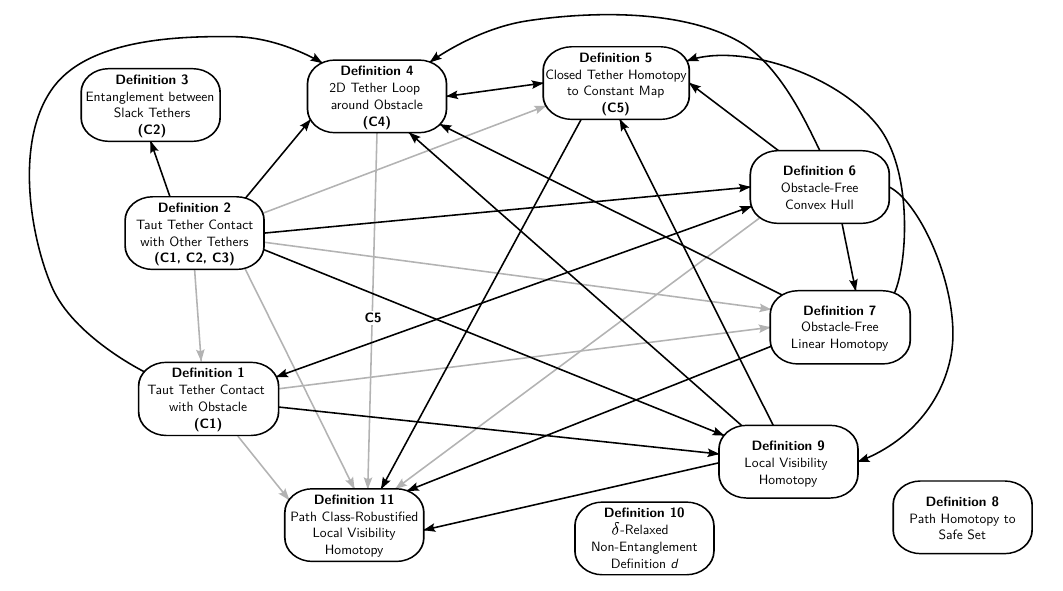}
	\caption{Graph showing the relationships between the non-entanglement definitions. Only the relationships in black are proven in this appendix. The ones in light gray can be derived by concatenation of other relationships. The relationships shown in this graph yield those reported in Table \ref{tab:definition comparison}.}
	\label{fig:definitions_relationships_graph}
\end{figure*}

\begin{proof}[Proof of Def. \ref{def:tether_contact_with_obstacle} $\implies$ Def. \ref{def:local_looping}]
	Let $\gamma$ be a taut tether in a 2D environment that is not entangled according to Definition \ref{def:tether_contact_with_obstacle}, i.e., such that $\gamma=l_{\x{a}, \x{r}}$.
	Since $\gamma$ coincides with a straight line segment, it does not contain any loop. This means that $\gamma(s_1) = \gamma(s_2) \iff s_1 = s_2$, and so the condition $\gamma_{[s_1, s_2]} \sim \gamma(s_1)$ of Definition \ref{def:local_looping} is always satisfied since $\gamma_{[s_1, s_2]} = \gamma(s_1)$. Thus $\gamma$ is not entangled according to Definition \ref{def:local_looping}.
\end{proof}

\begin{proof}[Proof of Def. \ref{def:tether_contact_with_obstacle} $\implies$ Def. \ref{def:obstacle_free_conv_hull}]
	Let $\gamma$ be a taut tether configuration that is not entangled according to Definition \ref{def:tether_contact_with_obstacle}, i.e., such that $\gamma=l_{\x{a}, \x{r}}$.
	By definition, the tether configuration $\gamma$ lies in $\mathcal{X}_\text{free}$, i.e., $\gamma\cap\interior\mathcal{O}=\emptyset$. 
	Since in the case of a straight line tether we have $\conv(l_{\x{a}, \x{r}}) = l_{\x{a}, \x{r}}$, it follows that $\conv(\gamma)\cap\interior\mathcal{O}=\emptyset$, which means that the tether is not entangled according to Definition \ref{def:obstacle_free_conv_hull}.
\end{proof}

\begin{proof}[Proof of Def. \ref{def:tether_contact_with_obstacle} $\implies$ Def. \ref{def:local_visibility_homotopy}]
	Let $\gamma$ be a taut tether configuration that is not entangled according to Definition \ref{def:tether_contact_with_obstacle}, i.e., such that $\gamma=l_{\x{a}, \x{r}}$. 
	For any $s_1, s_2 \in [0,1], s_1 \leq s_2$, we have $\gamma_{[s_1, s_2]} = l_{\gamma(s_1), \gamma(s_2)}$. It follows that $\gamma_{[s_1, s_2]} \sim l_{\gamma(s_1), \gamma(s_2)}$, which means that $\gamma$ is not entangled according to Definition~\ref{def:local_visibility_homotopy}.
\end{proof}

\begin{proof}[Proof of Def. \ref{def:taut_tether_contact_with_other_tether} $\implies$ Def. \ref{def:risk_of_entanglement}]
	Let $\mathcal{I}$ be a set of robots in a 3D environment, and $\gamma_i, i\in\mathcal{I}$ be a taut tether that is not entangled according to Definition \ref{def:taut_tether_contact_with_other_tether}, i.e., such that $\gamma_i = l_{x_{\text{a}, i}, x_{\text{r}, i}}$.
	Given the properties of the environment, the signature of each tether can be computed on a 2D projection of the environment, as detailed in \cite{cao2023neptune}.
	For every obstacle $O_i, i\in\{1, \ldots, m\}$ in the environment, a point $\hat{x}_i$ is then selected in its interior and two rays $\zeta_i = \hat{x}_i + sv$ and $\zeta_i = \hat{x}_i - sv$ are generated from it, where $v$ is a unit direction vector that is selected at the beginning of this process and used for every obstacle.
	The rays are added to the set of representative curves that will be used to compute the signature.
	Then, for every other robot $j\in\mathcal{I}\setminus\{i\}$, a piecewise-linear approximation of the tether is computed (see \cite[p.~2791]{cao2023neptune}). Each segment composing this approximation is added to the set of representative curves.
	Finally, the signature of $\gamma_i$ is computed.
	Since the path $\gamma_i$ coincides with the straight line segment $l_{x_{\text{a}, i}, x_{\text{r}, i}}$, it cannot intersect with any of the rays or straight line segments more than once. 
	Therefore $\gamma_i$ is not entangled according to Definition \ref{def:risk_of_entanglement}.
\end{proof}

\begin{proof}[Proof of Def. \ref{def:taut_tether_contact_with_other_tether} $\implies$ Def. \ref{def:local_looping}]
	Same proof as that of Def. \ref{def:tether_contact_with_obstacle} $\implies$ Def. \ref{def:local_looping} for $\gamma_i = l_{x_{\text{a},i},x_{\text{r}, i}}$.
\end{proof}

\begin{proof}[Proof of Def. \ref{def:taut_tether_contact_with_other_tether} $\implies$ Def. \ref{def:obstacle_free_conv_hull}]
	Same proof as that of Def. \ref{def:tether_contact_with_obstacle} $\implies$ Def.  \ref{def:obstacle_free_conv_hull} for $\gamma_i = l_{x_{\text{a},i},x_{\text{r}, i}}$.
\end{proof}

\begin{proof}[Proof of Def. \ref{def:taut_tether_contact_with_other_tether} $\implies$ Def. \ref{def:local_visibility_homotopy}]
	Same proof as that of Def. \ref{def:tether_contact_with_obstacle} $\implies$ Def.  \ref{def:local_visibility_homotopy} for $\gamma_i = l_{x_{\text{a},i},x_{\text{r}, i}}$.
\end{proof}

\begin{proof}[Proof of Def. \ref{def:local_looping} $\implies$ Def. \ref{def:closed_tether_homotopy_to_constant_map}]
	Let $\gamma$ be a closed tether configuration that is not entangled according to Definition \ref{def:local_looping}, i.e., such that $\gamma_{[s_1, s_2]}\sim \gamma(s_1), \forall s_1, s_2 \in [0,1]$ such that $\gamma(s_1)=\gamma(s_2)$.
	From this assumption and the fact that $\gamma$ is closed we have $\gamma_{[0, 1]} \sim \gamma(0)$, i.e., $\gamma \sim \x{a}$.
	Thus, $\gamma$ is not entangled according to Definition \ref{def:closed_tether_homotopy_to_constant_map}.
\end{proof}

\begin{proof}[Proof of Def. \ref{def:closed_tether_homotopy_to_constant_map} $\implies$ Def. \ref{def:local_looping}]
	Let $\gamma$ be a 2D closed tether configuration that is not entangled according to Definition \ref{def:closed_tether_homotopy_to_constant_map}, i.e., such that $\gamma \sim \x{a}$.
	For a 2D closed path $\gamma$ to be homotopic to a constant map there cannot be any obstacle being encircled by $\gamma$ \cite{bhattacharya2012topological}. Therefore, in any loop $\gamma_{[s_1, s_2]}$ such that $\gamma(s_1) = \gamma(s_2), s_1 \neq s_2$ there cannot be any obstacle as well. 
	This means that $\gamma_{[s_1, s_2]} \sim \gamma(s_1), \forall s_1, s_2\in[0,1]$, and therefore $\gamma$ is not entangled according to Definition \ref{def:local_looping}.
\end{proof}

\begin{proof}[Proof of Def. \ref{def:closed_tether_homotopy_to_constant_map} $\implies$ Def. \ref{def:path_class_relaxed_local_visibility_homotopy}]
	Let $\gamma$ be a closed path such that $\gamma(0) = \gamma(1) = x_\text{a}$ and $\gamma\sim x_\text{a}$. It is straightforward that the constant map $x_a$ satisfies the Local Visibility Homotopy definition (Definition \ref{def:local_visibility_homotopy}). Therefore, there exists a path that is not entangled according to the Local Visibility Homotopy definition and that is in the same path homotopy class as $\gamma$. Thus, $\gamma$ is not entangled according to Definition \ref{def:path_class_relaxed_local_visibility_homotopy}.
\end{proof}

\begin{proof}[Proof of Def. \ref{def:obstacle_free_conv_hull} $\implies$ Def. \ref{def:tether_contact_with_obstacle}]
	Let $\gamma$ be a taut path that is not entangled with respect to Definition \ref{def:obstacle_free_conv_hull}, i.e., such that $\conv(\gamma) \cap \interior\mathcal{O} = \emptyset$.
	It is easy to see that the taut path $\gamma$ must coincide with the straight-line path $l_{\x{a}, \x{r}}$. In fact, any other taut path must contain a bend around some obstacle, which violates the assumption that $\conv(\gamma)\cap\interior\mathcal{O}=\emptyset$. 
	Thus, $\gamma$ is also not entangled with respect to Definition \ref{def:tether_contact_with_obstacle}.
\end{proof}

\begin{proof}[Proof of Def. \ref{def:obstacle_free_conv_hull} $\implies$ Def. \ref{def:local_looping}]
	Let $\gamma$ be a tether configuration that is not entangled according to Definition \ref{def:obstacle_free_conv_hull}, i.e., such that $\conv(\gamma) \cap\interior\mathcal{O}=\emptyset$. If $\gamma(s_1) \neq \gamma(s_2), \forall s_1, s_2$, then $\gamma$ is trivially not entangled according to Definition \ref{def:local_looping}. 
	Otherwise, there exist some $s_1, s_2$ for which $\gamma(s_1)=\gamma(s_2)$. Since $\conv(\gamma_{[s_1, s_2]}) \subseteq \conv(\gamma)$, we have $\conv(\gamma_{[s_1, s_2]})\cap\interior\mathcal{O}=\emptyset, \forall s_1, s_2 \in [0,1]$.
	From Lemma \ref{lemma:path_homotopy_in_convex_set} applied with $\mathcal{Y} = \conv(\gamma_{[s1,s2]})$ it holds that $\gamma_{[s_1, s_2]} \sim \gamma(s_1)$, which means that $\gamma$ is not entangled according to Definition \ref{def:local_looping}.
\end{proof}

\begin{proof}[Proof of Def. \ref{def:obstacle_free_conv_hull} $\implies$ Def. \ref{def:closed_tether_homotopy_to_constant_map}]
	Let $\gamma$ be a closed tether configuration (i.e., such that $\gamma(0)=\gamma(1)$) that is not entangled according to Definition \ref{def:obstacle_free_conv_hull}, i.e., for which $\conv(\gamma) \cap\interior\mathcal{O}=\emptyset$.
	From Lemma \ref{lemma:path_homotopy_in_convex_set} applied with $\mathcal{Y} = \conv(\gamma)$ we have $\gamma \sim \gamma(0)$, thus $\gamma$ is not entangled according to Definition \ref{def:closed_tether_homotopy_to_constant_map}.
\end{proof}

\begin{proof}[Proof of Def. \ref{def:obstacle_free_conv_hull} $\implies$ Def. \ref{def:obstacle_free_linear_homotopy}]
	Let $\gamma$ be a tether configuration that is not entangled according to Definition \ref{def:obstacle_free_conv_hull}, i.e., for which $\conv(\gamma) \cap\interior\mathcal{O}=\emptyset$.
	The straight line segment $l_{\gamma(s), \x{a}}$ consists of all the convex combinations of the points $\gamma(s)$ and $\x{a}$, and therefore belongs to $\conv(\gamma)$. This holds for all the points of $\gamma$, i.e., $l_{\gamma(s), \x{a}} \subseteq \conv(\gamma), \forall s\in[0,1]$. Since $\conv(\gamma)\cap\interior\mathcal{O}=\emptyset$, then also the linear homotopy defined in \eqref{eq:linear_homotopy} has empty intersection with the interior of the obstacle region $\mathcal{O}$, i.e., $H(s,t)\in\mathcal{X}_\text{free}, \forall s,t\in[0,1]$.
	Thus, $\gamma$ is also not entangled with respect to Definition \ref{def:obstacle_free_linear_homotopy}.
\end{proof}

\begin{proof}[Proof of Def. \ref{def:obstacle_free_conv_hull} $\implies$ Def. \ref{def:local_visibility_homotopy}]
	Let $\gamma$ be a tether configuration that is not entangled according to Definition \ref{def:obstacle_free_conv_hull}, i.e., for which $\conv(\gamma) \cap\interior\mathcal{O}=\emptyset$.
	For every pair of scalars $s_1, s_2 \in [0,1], s_2\geq s_1$, the path $\gamma_{[s_1, s_2]}$ belongs to $\conv(\gamma)$ and so does the straight line segment $l_{\gamma(s_1), \gamma(s_2)}$. From Lemma \ref{lemma:path_homotopy_in_convex_set} we have $\gamma_{[s_1, s_2]} \sim l_{\gamma(s_1), \gamma(s_2)}$. Since this holds for all $s_1, s_2 \in [0,1]$, $\gamma$ is not entangled according to Definition \ref{def:local_visibility_homotopy}. 
\end{proof}

\begin{proof}[Proof of Def. \ref{def:obstacle_free_linear_homotopy} $\implies$ Def. \ref{def:local_looping}]
	Let $\gamma$ be a tether configuration that is not entangled according to Definition \ref{def:obstacle_free_linear_homotopy}, i.e., for which $l_{\gamma(s), \x{a}} \cap \interior\mathcal{O} = \emptyset, \forall s\in[0,1]$.
	Given any loop $\gamma_{[s_1, s_2]}$ in the tether $\gamma$ such that $\gamma(s_1)=\gamma(s_2)$, for $\gamma$ to be not entangled with respect to Definition \ref{def:obstacle_free_linear_homotopy} there cannot be any obstacle inside the area enclosed by $\gamma_{[s_1, s_2]}$, as otherwise there would be some point which violates the condition $l_{\gamma(s), \x{a}}\cap\interior\mathcal{O}=\emptyset$.
	Therefore, $\gamma_{[s_1, s_2]} \sim \gamma(s_1)$.
	Thus, $\gamma$ is not entangled according to Definition \ref{def:local_looping}.
\end{proof}

\begin{proof}[Proof of Def. \ref{def:obstacle_free_linear_homotopy} $\implies$ Def. \ref{def:closed_tether_homotopy_to_constant_map}]
	Let $\gamma$ be a tether configuration that is not entangled according to Definition \ref{def:obstacle_free_linear_homotopy}, i.e., for which $l_{\gamma(s), \x{a}} \cap \interior\mathcal{O} = \emptyset, \forall s\in[0,1]$.
	The existence of a linear homotopic mapping $H$ between $\gamma$ and $\x{a}$ directly implies that $\gamma$ is path-homotopic to the constant map $\x{a}$.
	Thus, $\gamma$ is not entangled according to Definition \ref{def:closed_tether_homotopy_to_constant_map}.
\end{proof}

\begin{proof}[Proof of Def. \ref{def:obstacle_free_linear_homotopy} $\implies$ Def. \ref{def:path_class_relaxed_local_visibility_homotopy}]
	Let $\gamma$ be a tether configuration that is not entangled according to Definition \ref{def:obstacle_free_linear_homotopy}, i.e., for which
	the linear homotopic map $H$ defined in \eqref{eq:linear_homotopy} has an empty intersection with the interior of the obstacle region. By definition of $H$, the concatenation $\gamma \diamond l_{x_\text{a}, x_\text{r}}^\text{reverse}$ is null-homotopic. Therefore, by Lemma \ref{lemma:null_homotopic_loops}, $\gamma\sim l_{x_\text{a}, x_\text{r}}$.
	The straight-line segment $l_{x_\text{a}, x_\text{r}}$ satisfies Definition \ref{def:local_visibility_homotopy}, as the straight-line segment between any two points of $l_{x_\text{a}, x_\text{r}}$ is path-homotopic to itself.
	This means that there exists a path in the same path class of $\gamma$ that is not entangled according to Definition \ref{def:local_visibility_homotopy}.
	Thus, $\gamma$ is not entangled according to Definition \ref{def:path_class_relaxed_local_visibility_homotopy}.
\end{proof}

\begin{proof}[Proof of Def. \ref{def:local_visibility_homotopy} $\implies$ Def. \ref{def:local_looping}]
	Let $\gamma$ be a 2D tether configuration that is not entangled according to Definition \ref{def:local_visibility_homotopy}.
	For any loop, i.e., for any path $\gamma_{[s_1, s_2]}$ such that $\gamma(s_1) = \gamma(s_2)$, it holds that $\gamma_{[s_1, s_2]} \sim \gamma(s_1)$, which is obtained by \eqref{eq:local_visibility_homotopy} with $l_{\gamma(s_1), \gamma(s_2)} = \gamma(s_1)$.
	Thus, $\gamma$ is not entangled according to Definition \ref{def:local_looping}.
\end{proof} 

\begin{proof}[Proof of Def. \ref{def:local_visibility_homotopy} $\implies$ Def. \ref{def:closed_tether_homotopy_to_constant_map}]
	Let $\gamma$ be a closed tether configuration that is not entangled according to Definition \ref{def:local_visibility_homotopy}.
	For $s_1=0, s_2=1$ we have $l_{\gamma(0), \gamma(1)} = \x{a}$. Since $\gamma$ is not entangled according to Definition \ref{def:local_visibility_homotopy} it holds that $\gamma_{[0, 1]} = \gamma \sim \x{a}$.
	Thus,  $\gamma$ is not entangled according to Definition \ref{def:closed_tether_homotopy_to_constant_map}.
\end{proof}

\begin{proof}[Proof of Def. \ref{def:local_visibility_homotopy} $\implies$ Def. \ref{def:path_class_relaxed_local_visibility_homotopy}]
	Let $\gamma$ be a tether configuration that is not entangled according to Definition \ref{def:local_visibility_homotopy}.
	Definition \ref{def:path_class_relaxed_local_visibility_homotopy} states that a tether configuration is not entangled if it is path-homotopic to another tether configuration that is not entangled according to Definition \ref{def:local_visibility_homotopy}. Since a path is always path-homotopic to itself, then $\gamma$ is also not entangled according to Definition \ref{def:path_class_relaxed_local_visibility_homotopy}.
\end{proof}

\bibliographystyle{IEEEtran}
\bibliography{references}


\begin{IEEEbiography}[
	{\includegraphics
		[width=1in, height=1.25in, clip, keepaspectratio]{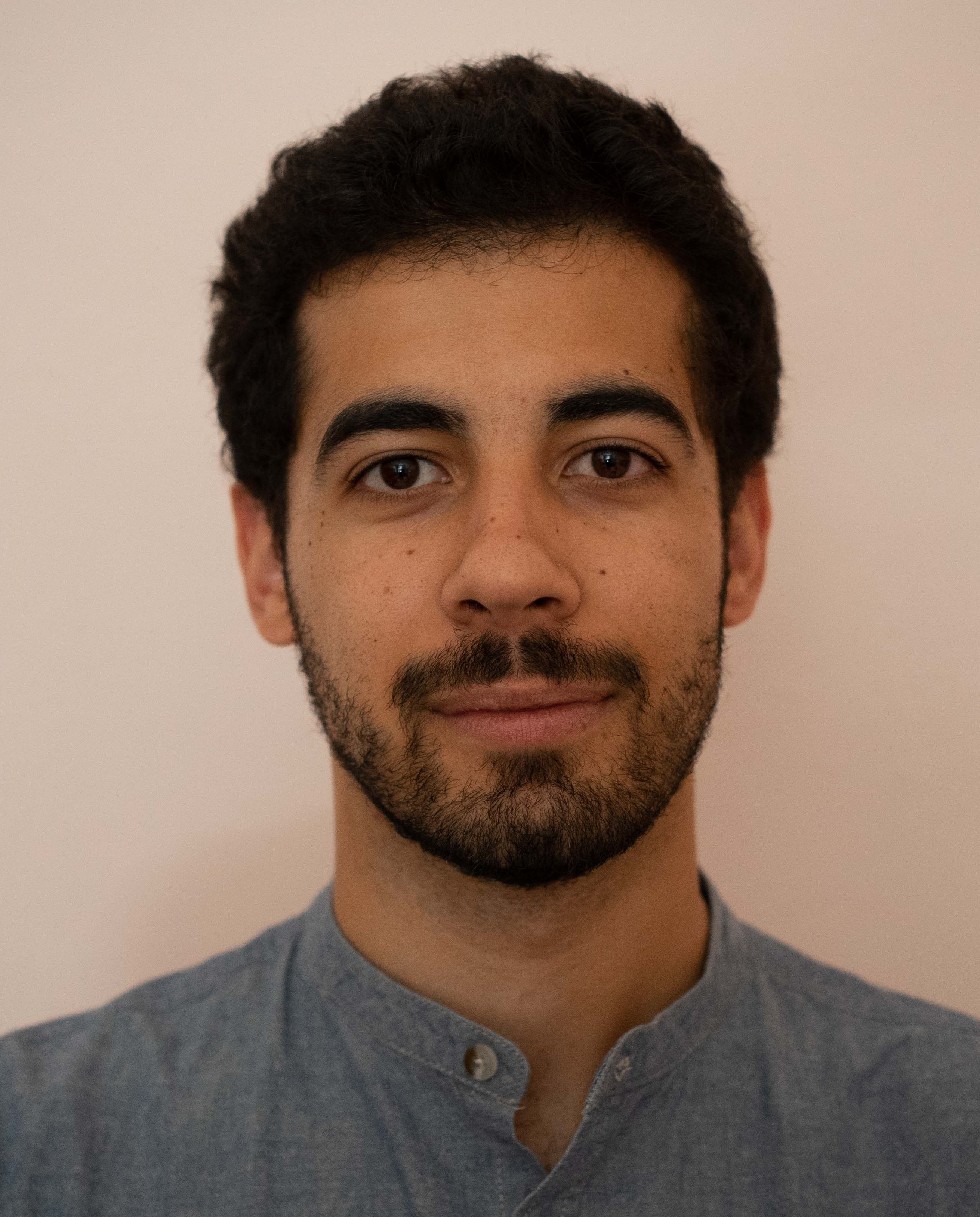}}
	]
	{Gianpietro Battocletti}
	received the B.Sc. degree from the University of Trento, Trento, Italy, in 2018, and the M.Sc. degree from Politecnico di Torino, Turin, Italy, in 2021.
	He is currently a Ph.D. candidate at the Delft Center for Systems and Control, Delft University of Technology, The Netherlands.
	His research interests include planning and control of mobile robots, with a particular focus on tethered robots, and cooperative control of multi-agent systems.
\end{IEEEbiography}

\begin{IEEEbiography}[
	{\includegraphics
		[width=1in, height=1.25in, clip, keepaspectratio]{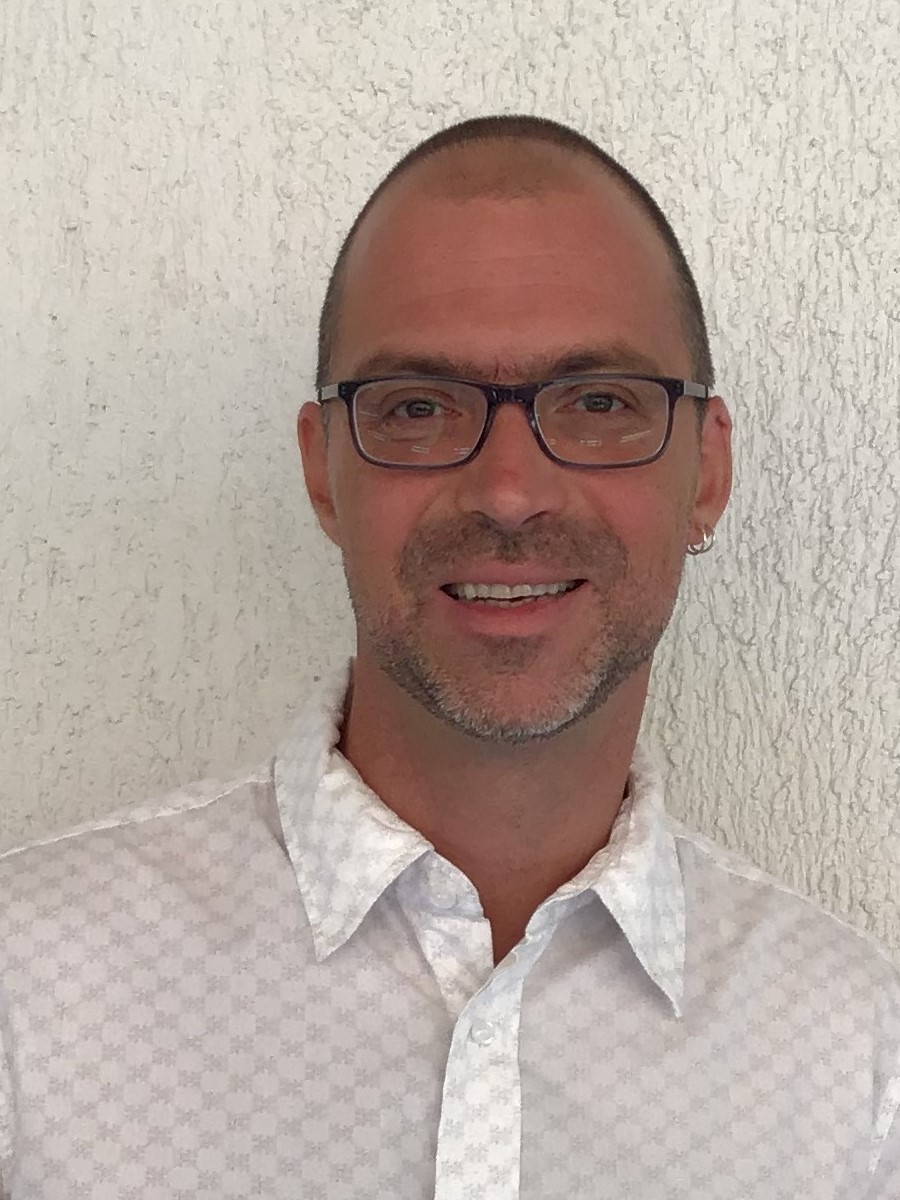}}
	]
	{Dimitris Boskos} 
	(M' 15) was born in Athens, Greece in 1981. He received a Diploma degree in Mechanical Engineering from the National Technical University of Athens (NTUA), Greece, in 2005, the M.Sc. degree in Applied Mathematics from the NTUA in 2008 and the Ph.D. degree in Applied Mathematics from the NTUA in 2014. Between August 2014 and August 2018, he has been a Postdoctoral Researcher with the Department of Automatic Control, School of Electrical Engineering, Royal Institute of Technology (KTH), Stockholm, Sweden. Between August 2018 and August 2020, he has been a Postdoctoral Researcher with the Department of Mechanical and Aerospace Engineering, University of California, San Diego, CA, USA. Since October 2020, he is an Assistant Professor with the Delft Center for Systems and Control, Delft University of Technology, Delft, The Netherlands. His research interests include distributionally robust optimization, distributed control of multi-agent systems, formal verification, and nonlinear observer design.
\end{IEEEbiography}

\begin{IEEEbiography}[
	{\includegraphics
		[width=1in, height=1in, clip, keepaspectratio]{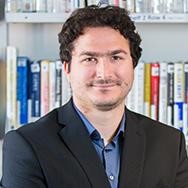}}
	]
	{Domagoj Toli\'{c}}
	is with RIT Croatia in Dubrovnik and the LARIAT – Laboratory for Intelligent Autonomous Systems. In 2012, he received a PhD in control systems from University of New Mexico, Albuquerque, NM. He was a postdoctoral researcher at University of Zagreb, Technical University of Munich and University of Dubrovnik. His professional interests are stability and estimation of nonlinear networked control systems with applications in robotics and multi-agent systems.
\end{IEEEbiography}

\begin{IEEEbiography}[
	{\includegraphics
		[width=1in, height=1in, clip, keepaspectratio]{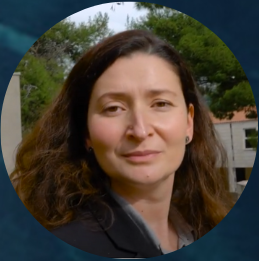}}
	]
	{Ivana Palunko} is an associate professor at the Department of Electrical Engineering and Computing, University of Dubrovnik. 
	She earned her Ph.D. in 2012 in Controls Systems from University of New Mexico, USA under supervision of prof. Rafael Fierro. 
	From 2012 until 2014 she was a postdoctoral researcher in the Research Centre for Advanced Cooperative Systems (ACROSS) at the Department of Control and Computer Engineering, University of Zagreb, Faculty of Electrical Engineering and Computing (UNIZG-FER). In 2013/2014 she was a visiting researcher at Information-oriented Control, TU M\"unchen, Germany under supervision of prof. Sandra Hirche. 
	In 2014 she joined the University of Dubrovnik as an assistant professor. Since 2019 she is the head of the Laboratory for Intelligent Autonomous Systems (LARIAT) at the University of Dubrovnik. 
	Her main research interests are in optimal adaptive control, machine learning and artificial intelligence with applications in robotics and multi-agent systems.
\end{IEEEbiography}

\begin{IEEEbiography}[
	{\includegraphics
		[width=1in, height=1.25in, clip, keepaspectratio]{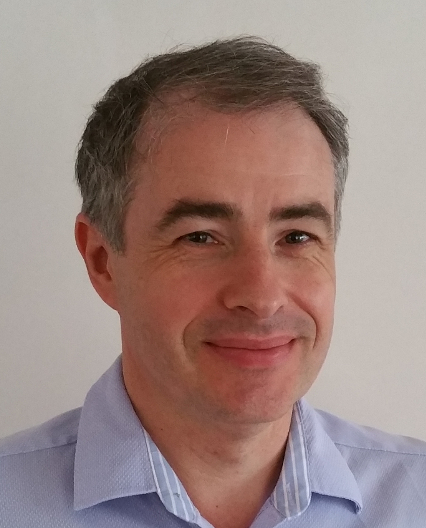}}
	]
	{Bart De Schutter} 
	(IEEE member since 2008, senior member since 2010, fellow since 2019) received the PhD degree in Applied Sciences (summa cum laude with congratulations of the examination jury) in 1996, at K.U.Leuven, Belgium. Currently, he is a full professor and head of department at the Delft Center for Systems and Control of Delft University of Technology in Delft, The Netherlands. 
	\par
	Bart De Schutter is senior editor of the \textsc{Ieee Transactions on Intelligent Transportation Systems}. 
	His current research interests include integrated learning- and optimization-based control, multi-level and multi-agent control of large-scale hybrid systems, and machine learning for decision making, with applications in intelligent transportation systems, smart energy systems, and underwater robotics.
\end{IEEEbiography}
\vfill

\EOD
\end{document}